\documentclass[twoside,11pt]{article}

\usepackage[nohyperref,preprint]{jmlr2e} 
\usepackage[hidelinks]{hyperref}

\usepackage{mymacros}
\usepackage{longtable}
\usepackage{microtype}
\usepackage{graphicx}
\usepackage{booktabs}
\usepackage{xcolor}
\usepackage{url}
\usepackage{nicefrac}
\usepackage{algorithm2e}
\usepackage{comment}
\usepackage{enumitem}
\usepackage{tikz}
\usepackage{tabularx}
\usepackage{multirow}
\usepackage{subcaption}
\usepackage{placeins}
\usepackage[makeroom]{cancel}
\usepackage{dsfont}
\usepackage{adjustbox}
\usepackage{array}
\usepackage{amsmath}
\usepackage{amssymb}
\usepackage{mathtools}
\usepackage{amsfonts}
\usepackage{bbm}
\usepackage{bm}
\usepackage{thmtools}
\usepackage{thm-restate}
\usepackage[capitalize,noabbrev]{cleveref}
\usepackage{amsmath}
\usepackage{lastpage}
\jmlrheading{X}{2024}{1-\pageref{LastPage}}{9/24; Revised XX/YY}{XX/YY}{21-0000}{S. Drago, M. Mussi and A. M. Metelli}

\ShortHeadings{A Refined Analysis of UCBVI}{Drago, Mussi and Metelli}
\firstpageno{1}

\allowdisplaybreaks[4]

\pgfplotsset{compat=1.18}

\begin{document}

\title{A Refined Analysis of UCBVI}

\author{\name Simone Drago \email simone.drago@polimi.it \\
\name Marco Mussi \email marco.mussi@polimi.it \\
\name Alberto Maria Metelli \email albertomaria.metelli@polimi.it \\
\addr Politecnico di Milano \\
Piazza Leonardo da Vinci 32, Milan, 20133, Italy}

\editor{My editor}

\maketitle

\begin{abstract}%
In this work, we provide a refined analysis of the \texttt{UCBVI} algorithm~\citep{azar2017minimax}, improving both the bonus terms and the regret analysis. Additionally, we compare our version of \texttt{UCBVI} with both its original version and the state-of-the-art \texttt{MVP} algorithm. Our empirical validation demonstrates that improving the multiplicative constants in the bounds has significant positive effects on the empirical performance of the algorithms. 
\end{abstract}

\section{Introduction}
\label{sec:intro}

We focus on the problem of \emph{finite-horizon tabular RL}, where the statistical complexity is characterized by a well-established lower bound of the order $\Omega ( \sqrt{HSAT} )$, where $S$ is the number of states, $A$ is the number of actions, $H$ is the horizon of the episode, and $T=HK$ where $K$ is the number of episodes~\citep{domingues2021episodic}. State-of-the-art learning algorithms match the lower bound up to logarithmic factors. \texttt{UCBVI}~\citep{azar2017minimax} matches this lower bound by combining the classical value iteration approach with the \emph{optimism in the face of uncertainty} mechanism, achieving $\widetilde{\BigO}(\sqrt{HSAT})$ under the condition $T \ge \Omega (H^3 S^3 A)$. Recently, \texttt{MVP}~\citep{zhang2024settling} overcome this limitation by employing the so-called \emph{doubling trick}, ensuring order-optimal regret (up to logarithmic factors) for every time horizon. However, this improvement comes at the expense of significantly larger constant factors, which may imply poor empirical performances (see Section~\ref{sec:experiments}). Consequently, \ucbvi remains a competitive and practical solution for finite-horizon tabular RL.

In this work, we improve both the bound and the analysis of the \ucbvi algorithm, with particular attention to its advanced form with Bernstein-Freedman optimistic bonus. Our goal is to design an exploration bonus that is as tight as possible and to conduct a regret upper bound analysis that minimizes also constants. 
\section{Setting}
\label{sec:preliminaries}

In this section, we introduce the notation and the setting we consider in the rest of the work.

\paragraph{Notation.}
Given a measurable set $\mathcal{X}$, we denote with $\Delta(\mathcal{X})$ the set of probability measures over $\mathcal{X}$, and with $|\mathcal{X}|$ it cardinality. For $n \in \mathbb{N}$, we denote the set $\{1, \ldots, n\}$ as $\dsb{n}$. We denote the L1 norm of a vector as $\| \cdot \|_1$. We denote the indicator function of event $x$ as $\mathbb{I}\{x\}$.

\paragraph{Markov Decision Processes.}
An undiscounted, episodic Markov Decision Process \citep[MDP, ][]{puterman1990} is a tuple $\mathcal{M} \coloneqq (\Ss, \As, P, R, H)$. In this tuple, $\Ss$ is the state space, $\As$ is the action space, $P : \Ss \times \As \to \Delta(\Ss)$ represents the state transition probability, $R : \Ss \times \As \to \mathbb{R}$ represents the reward function, and $H \in \mathbb{N}_{\geq 1}$ is the length of each episode.\footnote{Let $x, y \in \Ss$ and $a \in \As$, we denote as $P(y|x,a)$ the probability of observing $y$ as the next state after playing action $a$ in state $x$, and $R(x,a)$ the reward obtained after playing action $a$ in state $x$.}

We assume the state space and the action space are finite sets, and we denote their cardinalities as $|\Ss| = S < + \infty$ and $|\As| = A < + \infty$. We assume the state transition probability and the reward do not depend on the stage. Moreover, we assume the reward to be deterministic, known, and bounded in $[0,1]$.

\paragraph{Interaction with the Environment.}
The agent interacts with the environment in a sequence of $K$ episodes. Denote as $x_{k,h}$ the state occupied by the agent at stage $h \in \dsb{H}$ of episode $k \in \dsb{K}$, and as $a_{k,h}^{\pi_k}$ the action played by the agent at stage $h$ of episode $k$ according to the policy $\pi_k$. We assume policies to be deterministic and stage-dependent, \ie $\pi : \Ss \times \dsb{H} \to \As$.

The interaction of the $k$-th episode is defined as follows. Starting from state $x_{k,1} \in \Ss$, the agent selects which action to play as $a_{k,h}^{\pi_k} = \pi_k (x_{k,h}, h)$ for every $h \in \dsb{H}$, and observes a sequence of next-states and rewards, until the end of the episode.

The function $\Vpi[]{h} : \Ss \to \mathbb{R}$ denotes the value function at stage $h \in \dsb{H}$, such that $\Vpi[]{h}(x)$ represents the expected sum of the $H-h$ returns received under policy $\pi$ starting from state $x\in\Ss$. Under the assumptions stated above, there exists a deterministic policy $\pi^*$ which attains the best possible value function $\Vstar{h}(x) \coloneqq \sup_{\pi} \Vpi[]{h}(x)$ for every state $x \in \Ss$.
We measure the performance of a learning algorithm $\mathfrak{A}$ after $K$ episodes by means of the \emph{cumulative regret}:
\begin{equation*}
    \textnormal{Reg}(\mathfrak{A}, K) \coloneqq \sum_{i=1}^K \Vstar{1}(x_{i,1}) - \Vpi[i]{1}(x_{i,1}).
\end{equation*}
We denote as $T=KH$ the total number of interactions.
\section{\texttt{UCBVI}}

In this section, we consider the \ucbvi algorithm, introduced in~\citep{azar2017minimax}. First, we provide a more compact (but equivalent) pseudocode of the algorithm in Algorithm~\ref{alg:ucbvi}. 

We start by initializing the counters needed in order to run the algorithm. Then, we start the continuous interaction procedure for every episode $k\in\dsb{K}$. For every episode, before starting the interaction, the algorithm computes the optimistic estimate of the value function. Such estimate starts by computing all the transition probabilities $\Pest_k (y | x, a)$ for every state, action and next state as:
\begin{align}
    \Pest_k (y | x, a) = \frac{N_k(x, a, y)}{\max\{1,N_k(x, a)\}},
\end{align}
where $N_k(x, a)$ is the number of times we play action $a\in\As$ in state $x\in\Ss$, and $N_k(x, a, y)$ is the number of times we do so and observe the next state $y\in\Ss$. Then, we compute the optimistic value iteration for finite horizon MDPs, starting from stage $H$ backward, where the optimistic $Q_{k,h} (x,a)$ is defined as:
\begin{align*}
    Q_{k,h} (x,a) = \min \{ Q_{k-1,h} (x,a), R (x,a) + \! \sum_{y \in \mathcal{S}} \! \Pest_k (y | x, a) \Vest{k,h+1}(y) + b_{k,h} (x, a) \}.
\end{align*}
This procedure mimics value iteration with an additive exploration term $b_{k,h} (x, a)$ which will be further characterized later in Section~\ref{sec:analysis}. Then, term $V_{k,h} (x)$ is computed as usual for value iteration:
\begin{align*}
    \Vest{k,h}(x) = \max_{a \in \mathcal{A}} Q_{k,h}(x, a).
\end{align*}
After that, we can start interacting with the environment and play greedily. More in detail, at every stage $h\in\dsb{H}$ of episode $k\in\dsb{K}$, we observe state $x_{k,h}$ and play the most promising action according to the optimistic estimate:
\begin{align*}
a_{k,h} \in \argmax_{a \in \mathcal{A}} Q_{k,h} (x_{k,h}, a).
\end{align*}
After having played $a_{k,h}$, we observe the reward $r_{k,h}$ and the next state $x_{k,h+1}$. Finally, we use the collected information to properly update counters.

\RestyleAlgo{ruled}
\LinesNumbered
\begin{algorithm}[t!]
\caption{\ucbvi.}\label{alg:ucbvi}
{\small
\textbf{Initialize}: $N_0(x, a, y) = 0$, $N_0 (x, a) = 0$, $N'_{0,h} (x) = 0, \ \forall (x, a, y) \in \mathcal{S} \times \mathcal{A} \times \mathcal{S}$ \\ 
\phantom{\textbf{Initialize}:} $Q_{0,h} (x,a) = H - h + 1, \ \forall (x, a, h) \in \mathcal{S} \times \mathcal{A} \times \dsb{H}$

\For{$k \in \dsb{K}$}{

    \vspace{.1cm}

    \textcolor{black!40!white}{\texttt{// Update the optimistic estimates for episode} $k$}
    
    Estimate $\Pest_k (y | x, a) = N_k(x, a, y)/ \max\{1,N_k(x, a)\}$ \label{alg:ucbvi:pest}
    
    Initialize $\Vest{k,H+1}(x) = 0$, $\forall x \in \mathcal{S}$
    
    \For{$h = \{ H, H-1, \ldots, 1 \}$}{
    
        \For{$x \in \mathcal{S}$}{

            \For{$a \in \mathcal{A}$}{
                $Q_{k,h} (x,a) = \min \{ Q_{k-1,h} (x,a), R (x,a) + \! \sum_{y \in \mathcal{S}} \! \Pest_k (y | x, a) \Vest{k,h+1}(y) + b_{k,h} (x, a) \}$ \label{alg:ucbvi:qvals}
            }
    
            $\Vest{k,h}(x) = \max_{a \in \mathcal{A}} Q_{k,h}(x, a)$ \label{alg:ucbvi:vvals}
                        
        }
    }

    \textcolor{black!40!white}{\texttt{// Interact with the environment for episode} $k$}
    
    Agent observes state $x_{k,1}$
    
    \For{$h \in \dsb{H}$}{

        Agent plays action $a_{k,h} \in \argmax_{a \in \mathcal{A}} Q_{k,h} (x_{k,h}, a)$ \label{alg:ucbvi:play}

        Environment returns reward $r_{k,h}$ and next state $x_{k,h+1}$ \label{alg:ucbvi:observe}

        Update for every $(x, a, y) \in \mathcal{S} \times \mathcal{A} \times \mathcal{S}$: \label{alg:ucbvi:updatefirst}

        \phantom{Update} $N_k(x, a, y) = N_{k-1}(x, a, y) + \mathds{1}\{x=x_{k,h}, a=a_{k,h}, y = x_{k,h+1}\}$ 

        \phantom{Update} $N_k(x, a) = N_{k-1}(x, a) + \mathds{1}\{x=x_{k,h}, a=a_{k,h}\}$

        \phantom{Update} $N_{k,h}'(x) = N_{k-1,h}'(x) + \mathds{1}\{x=x_{k,h}\}$ \label{alg:ucbvi:updatelast}
    }
}
}
\end{algorithm}

\section{A Refined \texttt{UCBVI} Analysis}
\label{sec:analysis}

In this section, we analyze the \ucbvi algorithm. We start in Section~\ref{sec:analysis:ch} by providing the regret upper bound for \ucbvi with \emph{Chernoff-Hoeffding} bonus (Theorem~\ref{thr:ucbvichUB}). Then, we consider the more elaborated \emph{Bernstein-Freedman} bonus, and we characterize the regret of \ucbvi we get when we consider such a bonus (Theorem~\ref{thr:ucbvibfUB}). Our contribution consists of refining the analysis to obtain tighter bonuses and, as a direct consequence, tighter regret bounds.

\subsection{Chernoff-Hoeffding}
\label{sec:analysis:ch}

Let us start with the \emph{Chernoff-Hoeffding} bound, taking the opportunity to correct some typos of the original analysis of~\citep{azar2017minimax}.\footnotetext[2]{We assume that, by definition, $\bonus_{k,H}(s,a) = 0$, as at the last stage, there is no need for exploration, and the rewards are deterministic.}
\addtocounter{footnote}{1}

\begin{restatable}[Regret for \ucbvi with Chernoff-Hoeffding bonus]{thr}{ucbvichUB}\label{thr:ucbvichUB}
Let $\delta \in (0,1)$. Considering:\footnotemark[2]
\begin{equation*}
    b_{k,h}(x,a) = \frac{2HL}{\sqrt{\max \{ N_{k}(x,a), 1 \}}}, 
\end{equation*}
then, \wp at least $1-\delta$, the regret of \ucbvich is bounded by:
\begin{equation*}
    \regch{K} \le 10 e HL \sqrt{SAT} + \frac{8}{3} e H^2 S^2 A L^2,
\end{equation*}
where $L = \ln{(5HSAT / \delta)}$. For $T \ge \Omega ( H^2 S^3 A )$, this bound translates to $\widetilde{\BigO}(H \sqrt{SAT})$.
\end{restatable}%
Theorem~\ref{thr:ucbvichUB} should be compared to Theorem~1 of~\citep{azar2017minimax}. Since the analysis is a refinement of the original analysis in terms of constants, the order of the regret does not change between the two theorems. However, our analysis provides a smaller value for the constants.\footnote{To the best of the authors' knowledge, the original analysis of~\citep{azar2017minimax} is missing a multiplicative $e$ factor in the regret bound.} Moreover, observe how the minimum value of $T$ for which the regret bound holds according to our analysis is $H$ times higher than the one reported in the original theorem. This is due to the fact that the minimum $T$ in the statement of Theorem~1 of~\citep{azar2017minimax} is incorrect, although the derivation in the appendix provides the same minimum value of $T$ we obtain.

\subsection{Bernstein-Freedman}
\label{sec:analysis:bf}

We now move to the \emph{Bernstein-Freedman} bonus.
\begin{restatable}[Regret for \ucbvi with Bernstein-Freedman bonus]{thr}{ucbvibfUB}\label{thr:ucbvibfUB}
Let $\delta \in (0,1)$. Considering:\footnotemark[2]  
\begin{align*}
    b_{k,h}(x,a) = & \underbrace{\sqrt{\frac{4L \Var_{y \sim \widehat{P}_k(\cdot | x,a)}(V_{k,h+1}(y))}{\max \{ N_{k}(x,a), 1 \}}}}_{\textup{\texttt{(A)}}} + \underbrace{\frac{7HL}{ 3 \max \{ N_{k}(x,a) - 1, 1 \}}}_{\textup{\texttt{(B)}}} \\ & \quad + \underbrace{\sqrt{\frac{4 \sum_{y \in \mathcal{S}} \left( \widehat{P}(y | x, a) \cdot \min \left\{ \frac{84^2 H^3 S^2 A L^2}{\max \{1, N_{k,h+1}'(y) \} } , H^2 \right\} \right)}{\max \{ N_{k}(x,a), 1 \}}}}_{\textup{\texttt{(C)}}},
\end{align*}
then, \wp at least $1-\delta$, the regret of \ucbvibf is bounded by:
\begin{equation*}
    \regbf{K} \le 24 e L \sqrt{HSAT} + 616 e H^2 S^2 A L^2 + 4 e \sqrt{H^2 TL},
\end{equation*}
where $L = \ln (5HSAT / \delta )$. For $T \ge \Omega ( H^3 S^3 A )$ and $SA \ge H$, this bound translates to $\widetilde{\BigO}(\sqrt{HSAT})$.
\end{restatable}

Theorem~\ref{thr:ucbvibfUB} should be compared to Theorem~2 of~\citep{azar2017minimax}. Again, as the analysis is a refinement in terms of constants, the order of the regret does not change. Moreover, also the minimum value of $T$ under which the regret bound matches the lower bound in unchanged between the two analyses. It is important to notice, however, that the constant terms of our analysis are strictly smaller than the ones of~\citep{azar2017minimax}.\footnotemark[3] In particular, comparing the bonus term $b_{k,h}$ presented in Theorem~\ref{thr:ucbvibfUB}  \wrt the one of \citep[][Algorithm~4]{azar2017minimax} we observe the following differences:
\begin{itemize}[topsep=1pt,noitemsep,leftmargin=15pt]
    \item in term \texttt{(A)}, we have a $\sqrt{4}$ multiplicative factor instead of $\sqrt{8}$;
    \item in term \texttt{(B)}, we have a $7$ multiplicative factor instead of $14$;
    \item in term \texttt{(C)}, we have a $\sqrt{4}$ multiplicative factor instead of $\sqrt{8}$ and, inside the minimum, a term multiplicative factor $84^2$ instead of $100^2$.
\end{itemize}
The reader shall refer to Appendices~\ref{apx:proof_CH} and \ref{apx:proof_BF} for the proofs of Theorems~\ref{thr:ucbvichUB} and \ref{thr:ucbvibfUB}, respectively. The derivations provided in the appendices closely follow the proofs of~\citep{azar2017minimax}, focusing on lowering the constant terms.
A full description of the notation employed throughout the paper is reported in Appendix~\ref{apx:notation}. Both proofs are conducted under the condition that concentration inequalities hold for the next state estimator and its variance. Those conditions fall under event $\mathcal{E}$, which is presented in Appendix~{B.4} of~\citep{azar2017minimax} and restated in Appendix~\ref{apx:event}. Finally, additional lemmas necessary to show the regret decomposition and to bound the summation of the term it comprises are demonstrated in Appendix~\ref{apx:lemmas}.
\section{Numerical Validation}
\label{sec:experiments}

In this section, we numerically compare the performances of \ucbvi, both with the Chernoff-Hoeffding and Bernstein-Freedman bonuses of~\citep{azar2017minimax} and with the improved Bernstein-Freedman bonus of this paper, against the \mvp \citep{zhang2024settling} algorithm.

In order to fairly compare to the \mvp algorithm, all the $N_h(x,a)$ terms are considered as $N(x,a)$, removing the discriminant of the stage from the algorithm, and the $c_2$ constant (which refers to the uncertainty in the estimation of the rewards) is set to $0$, to remove the exploration factor needed due to the stochasticity of the reward in the original paper.

\subsection{Illustrative Environments}

As a first experimental evaluation, we consider a set of illustrative environments. We consider an MDP with parameters $S=3$, $A=3$, $H \in \{5, 10\}$, and we consider a number of episodes $K \in \{10^5, 10^6\}$. We evaluate each experiment by averaging over 10 runs. In each run, the rewards and transition probabilities of the MDP are randomly generated. Then, the clairvoyant optimum is calculated for the purpose of regret computation, and the algorithms are evaluated.

\begin{figure*}[t!]
    \centering

    \hspace{0.8cm}
    \subfloat[$H=5, \ K=10^5$.]{\resizebox{0.37\linewidth}{!}{\includegraphics{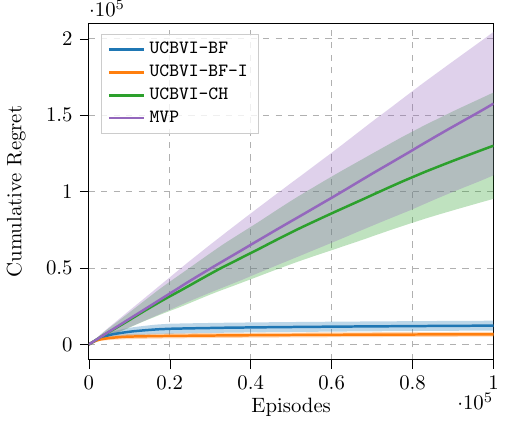}}}
    \hfill
    \subfloat[$H=5, \ K=10^6$.]{\resizebox{0.37\linewidth}{!}{\includegraphics{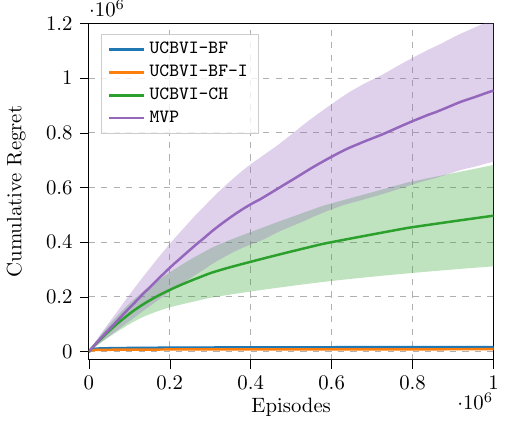}}}
    \hspace{1cm}

    \vspace{0.3cm}

    \hspace{0.8cm}
    \subfloat[$H=10, \ K=10^5$.]{\resizebox{0.37\linewidth}{!}{\includegraphics{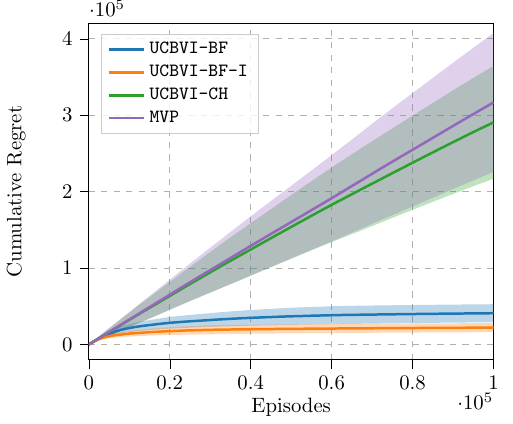}}}
    \hfill
    \subfloat[$H=10, \ K=10^6$.]{\resizebox{0.37\linewidth}{!}{\includegraphics{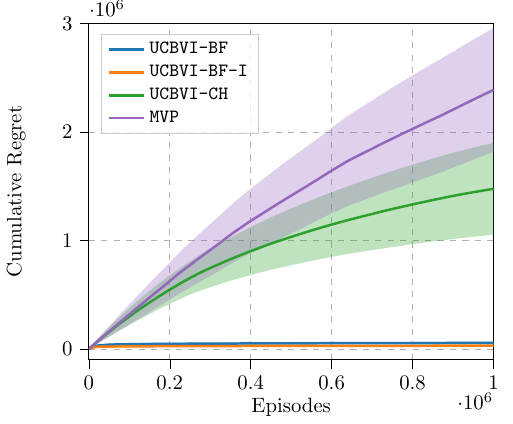}}}
    \hspace{1cm}
    
    \caption{Performances in terms of cumulative regret in toy environments with $S=3$ states and $A=3$ actions ($10$ runs, mean $\pm$ $95\%$ C.I.).}
    \label{fig:TODO2}
\end{figure*}




\paragraph{Results.}
Figure~\ref{fig:TODO2} represents the cumulative regret of the evaluated algorithms in the first experimental evaluation for different values of $H$ and $K$. From the figures, we can observe that \ucbvi with the Chernoff-Hoeffding bonus and \mvp begin to show a sub-linear regret for $K=10^6$, whereas both versions of \ucbvi with the Bernstein-Freedman bonus greatly outperform the other algorithms in all the evaluated scenarios. In particular, the use of a tighter Bernstein-Freedman bonus (\ucbvibfi) translates into a cumulative regret that is, although of the same order, lower than with the usage of a larger bonus (\ucbvibf), highlighting the importance of lower order terms and constants in empirical performance.

\subsection{RiverSwim}
We now consider the RiverSwim environment~\citep{strehl2008analysis}. This environment emulates a swimmer that has to swim against the current, where the agent has 2 options: \emph{(i)} to try to swim to the other side or \emph{(ii)} to turn back. In this scenario, the rewards and the transition probabilities are designed such that the optimal policy corresponds to trying to swim and reach the other side of the \quotes{river}. This is considered a challenging benchmark for exploration. We consider the scenario with $S=5$ and $H=10$. The reward model and the transition probability are designed such that the suboptimality gap between the optimal action and the other one in the initial state is very low ($\sim\!0.1$, with a scale of the problem in the order of $H=10$).

\paragraph{Results.}
Figure~\ref{fig:river} compares the results when using \mvp and \ucbvi in its original version (\ucbvibf) and the one we propose with tighter bounds (\ucbvibfi). \mvp confirms its poor empirical performance, failing to deliver a sublinear trend for the considered horizon. Instead, \ucbvi, in both versions, shows a clear sublinear trend, with the improved version (\ucbvibfi) with a cumulative regret approximately half of the original one (\ucbvibf).

\begin{figure}[t!]
\centering
\resizebox{0.37\linewidth}{!}{\includegraphics{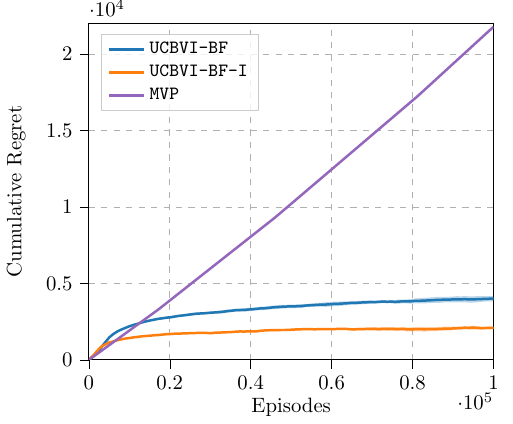}}
\caption{Performances in terms of cumulative regret in the RiverSwim environment with $S=5$ states and horizon $H=10$ ($4$ runs, mean $\pm$ $95\%$ C.I.).}
\label{fig:river}
\end{figure}

\begin{table}[t!]
    \centering
    \begin{tabular}{c|ccc}
        \hline 
        & Bonus ratio & Regret upper bound ratio &  Empirical regret ratio\\ \hline
        \texttt{CH} & $7/2$ & $2$ & - \\ \hline
        \texttt{BF} & $\sqrt{2}$ & $5/4$ & $1.87 \pm 0.03$ \\ \hline
    \end{tabular}
    \caption{Improvement ratios in the bonuses, regret upper bounds, and empirical regret between our analysis and the original of~\citep{azar2017minimax}.}
    \label{tab:ratios}
\end{table}

\section{Conclusions}
A summary of the improvements, expressed in terms of improvement ratios in the bonuses, regret upper bounds, and empirical regret, is reported in Table~\ref{tab:ratios}. First, we compare our versions of the \texttt{UCBVI} algorithms with the original ones from~\citep{azar2017minimax}. The algorithmic structure remains the same, though we re-derived the bonus terms to make them as tight as possible, resulting in an improvement of $7/2$ and $\sqrt{2}$ for the Chernoff-Hoeffding and Bernstein-Freedman bonuses, in the dominant terms, respectively. This reduction in over-exploration has significant empirical effects, as shown in Section~\ref{sec:experiments}, where, as reported in Table~\ref{tab:ratios}, we achieve an improvement in the empirical regret of $1.87$ times. Additionally, this impacts the regret analysis, where we were able to reduce the regret bound by a factor of $2$ and $5/4$ for the Chernoff-Hoeffding and Bernstein-Freedman bonuses, respectively, in terms of dominant terms. However, lower order terms also have an impact on the performance, and through a refined analysis we were able to reduce them by a factor of $\sim \! 90$ and $\sim \! 4$ for the Chernoff-Hoeffding and Bernstein-Freedman bonuses, respectively.

\clearpage
\vskip 0.2in
\bibliography{biblio}

\clearpage
\appendix

\setlength{\parindent}{0pt}
\section{Notation}
\label{apx:notation}

In this section, we collect the notation used throughout the main paper and the appendices.

\begin{table}[h!]
\centering
\resizebox{.98\linewidth}{!}{\begin{tabular}{| p{.20\textwidth} | p{.85\textwidth} |} 
\hline
\textbf{Symbol} & \textbf{Meaning} \\ \hline \hline
$\Ss$ & State space \\ 
$\As$ & Action space \\ 
$P$ & Transition distribution \\ 
$R$ & Reward function \\ 
$H$ & Length of the episode \\ 
$K$ & Total number of episodes \\ 
$T$ & Total number of steps \\ 
$T_k$ & Total number of steps up to episode $k$ \\ 
$S$ & Cardinality of the state space \\ 
$A$ & Cardinality of the action space \\ \hline
$x_{k,h}$ & State occupied at stage $h$ of episode $k$ \\ 
$a_{k, h}^{\pi}$ & Action played at stage $h$ of episode $k$ under policy $\pi$ \\ 
$R^\pi (x)$ & Reward obtained by playing according to policy $\pi$ in state $x$ \\ 
$N_k (x,a)$ & Number of visits to state-action pair $(x,a)$ up to episode $k$ \\ 
$N_k (x,a,y)$ & Number of transitions to state $y$ from $x$ after playing $a$, up to episode $k$ \\ 
$N'_{k,h}(x)$ & Number of visits to state $x$ at stage $h$ up to episode $k$ \\ 
$\Pest_{k}$ & Estimated transition distribution \\ 
$\bonus_{k,h}$ & Exploration bonus \\ 
$\bonus'_{k,h}(x)$ & $\min \{ \frac{84^2 H^3 S^2 A L^2 }{N'_{k,h}(x)} , H^2 \}$ \\ 
$\pi_k$ & Policy played during episode $k$ \\ 
$\pi^*$ & Optimal policy \\ 
$\Qest{k,h}$ & Optimistic state-action value function \\ 
$\Vstar{h}$ & Value function of the optimal policy at stage $h$ \\ 
$\Vpi[]{h}$ & Value function under policy $\pi$ at stage $h$ \\ 
$\Vest{k,h}$ & Optimistic estimator of the optimal value function at stage $h$ of episode $k$ \\ \hline
$\Vdiff{k,h}(x)$ & Regret in state $x$, at stage $h$ of episode $k$, following policy $\pi_k$ \\ 
$\Vdifftil{k, h}(x)$ & Pseudo-regret in state $x$, at stage $h$ of episode $k$, following policy $\pi_k$ \\ 
$\regch{k}$ & Regret of \ucbvi using Chernoff-Hoeffding bonus after $k$ episodes \\ 
$\regchtilde{k}$ & Pseudo-regret of \ucbvi using Chernoff-Hoeffding bonus after $k$ episodes \\ 
$\regbf{k}$ & Regret of \ucbvi using Bernstein-Freedman bonus after $k$ episodes \\ 
$\regbftilde{k}$ & Pseudo-regret of \ucbvi using Bernstein-Freedman bonus after $k$ episodes \\ \hline
$\mathcal{E}$ & Concentration inequalities event \\ 
$\Omega, \Omega_{k,h}$ & Optimism events \\ 
$\varepsilon$, $\overline{\varepsilon}$ & Martingale Difference Sequences \\ 
$\ktyp$, $\ktypx$ & Sets of typical episodes \\ 
$\histkh$ & History of the interactions up to, and including, stage $h$ of episode $k$ \\ \hline
$L$ & Logarithmic term $\ln (5HSAT/\delta)$ \\ 
$\VVpi{h}(x,a)$ & Next-state variance of $\Vpi{}$ \\ 
$\VVopt{h}$ & Next-state variance of $\Vstar{}$ \\ 
$\VVest{k,h}$ & Empirical next-state variance of $\Vest{k,h}$ \\ 
$\VVestopt{k,h}$ & Empirical next-state variance of $\Vstar{}$ \\ 
$\xi_{k,j} (x, a)$ & State-action wise model error ${\xi_{k,j} (x, a) \coloneqq \sum_{y \in \Ss} [ \Pest_k (y | x, a) - P(y | x, a)] \Vstar{h+1} (y)}$ \\ \hline
\end{tabular}}
\caption{Table of notation.}
\label{tab:notation}
\end{table}

Let us now state the definitions of the next-state variances employed in the analysis. We define the empirical next state variance of $\Vest{}$ as:

\begin{equation*}
    \VVest{k,h+1} (x,a) \coloneqq \Var_{y \sim \Pest_k (\cdot | x, a)} [\Vest{k,h+1} (y)].
\end{equation*}

We define the next state variance of $\Vstar{}$ as:

\begin{equation*}
    \VVopt{h+1} (x,a) \coloneqq \Var_{y \sim P (\cdot | x,a)} [\Vstar{h+1} (y)].
\end{equation*}

We define the next state variance of $\Vpi[]{}$ as:

\begin{equation*}
    \VVpi[]{h+1} (x,a) \coloneqq \Var_{y \sim P (\cdot | x,a)} [\Vpi[]{h+1} (y)].
\end{equation*}

Finally, we define the empirical next state variance of $\Vstar{}$ as:

\begin{equation*}
    \VVestopt{k,h+1} (x,a) \coloneqq \Var_{y \sim \Pest_k (\cdot | x,a)} [\Vstar{h+1} (y)].
\end{equation*}

Let us now state the definition of the Martingale Difference Sequences employed in the analysis:

\begin{align*}
    \varepsilon_{k,h} &\coloneqq P(\cdot | x_{k,h}, a_{k,h}^{\pi_k})^\transpose \Vdifftil{k,h+1}(\cdot) - \Vdifftil{k,h+1}(x_{k,h+1}), \\
    \overline{\varepsilon}_{k,h} &\coloneqq \sum_{y\in\Ss} P(y|x_{k,h},a_{k,h}^{\pi_k}) \sqrt{\frac{\mathbb{I}(y \in [y]_{k,h})}{N_k(x_{k,h},a_{k,h}^{\pi_k})P(y|x_{k,h},a_{k,h}^{\pi_k})}} \Vdifftil{k,h+1}(y) \\
    &\quad - \sqrt{\frac{\mathbb{I}(x_{k,h+1} \in [y]_{k,h})}{N_k(x_{k,h},a_{k,h}^{\pi_k})P(x_{k,h+1}|x_{k,h},a_{k,h}^{\pi_k})}} \Vdifftil{k,h+1}(x_{k,h+1}).
\end{align*}
\section{High Probability Events}
\label{apx:event}

In this section, we restate the high probability event $\mathcal{E}$ under which the concentration inequalities hold, presented in Appendix~B.4 of \citep{azar2017minimax}.

Event $\mathcal{E}$ is defined as:

\begin{align*}
    \mathcal{E} \coloneqq \mathcal{E}_{\Pest} \bigcap \bigcap_{\substack{k \in \dsb{K} \\ h \in \dsb{H} \\ x \in \Ss}} \bigg[ 
    & \evaz\left( \mathcal{F}_{\Vdifftil{}, k, h}, H, L \right) 
    \bigcap \evaz\left( \mathcal{F}'_{\Vdifftil{}, k, h}, \frac{1}{\sqrt{L}}, L \right)
    \bigcap \evaz\left( \mathcal{F}_{\Vdifftil{}, k, h, x}, H, L \right) \\
    & \bigcap \evaz\left( \mathcal{F}'_{\Vdifftil{}, k, h, x}, \frac{1}{\sqrt{L}}, L \right)
    \bigcap \evfr\left( \mathcal{G}_{\VV, k, h}, H^4T, H^3, L \right) \\
    & \bigcap \evfr\left( \mathcal{G}_{\VV, k, h, x}, H^5 N'_{k,h}(x), H^3, L \right)
    \bigcap \evaz\left( \mathcal{F}_{\bonus', k, h}, H^2, L \right) \\
    & \bigcap \evaz\left( \mathcal{F}_{\bonus', k, h, x}, H^2, L \right)
    \bigg]
\end{align*}

We refer the reader to Lemma~1 of \citep{azar2017minimax} for the proof that event $\mathcal{E}$ holds with high probability.
Let, for ease of reading $\overline{x} = x_{i,j}$, $\overline{x}' = x_{i,j+1}$, and $\overline{a}=a_{i,j}^{\pi_i}$
We now restate the definition of the events that compose $\mathcal{E}$:

\begin{align*}
    \mathcal{E}_{\Pest} \coloneqq \left\{ \Pest_k (y|x,a) \in \mathcal{P}(k,h,N_k(x,a),x,a,y), \forall k \in \dsb{K}, h \in \dsb{H}, (x,a,y)\in\Ss\times\As\times\Ss \right\},
\end{align*}

where $\mathcal{P}(k, h, n,x,a,y)$ is defined as the subset of the set of all probability distributions $\mathcal{P}$ over $\Ss$ such that:

\begin{align}
    \mathcal{P}(k, h, n,x,a,y) \coloneqq &\bigg\{ \widetilde{P}(\cdot | x,a) \in \mathcal{P} : \| \widetilde{P}(\cdot | x,a) - P(\cdot | x,a) \|_1 \le 2 \sqrt{\frac{SL}{n}}, \label{apx:event:1} \\
    \begin{split}
    \label{apx:event:2}
    & \left| \sum_{y\in\Ss} (\widetilde{P}(y|x,a) - P(y|x,a) ) \Vstar{h}(y) \right| \\
    &\le \min \left( \sqrt{\frac{2 \VVestopt{k,h+1}(x,a)L}{n}} + \frac{7 HL}{3 (n-1)}, \sqrt{\frac{2 \VVopt{h+1}(x, a) L}{n}} + \frac{2HL}{3n} \right),
    \end{split} \\
    & \left| \widetilde{P}(y|x,a) - P(y|x,a) \right| \le \sqrt{\frac{2 P(y|x,a) (1-P(y|x,a)) L}{n}} + \frac{2L}{3n}\bigg\},\label{apx:event:3}
\end{align}

where Equation~\eqref{apx:event:1} follows by applying the result of Theorem~{2.1} of \cite{weissman2003}, Equation~\eqref{apx:event:2} follows by applying both Bernstein's inequality \citep[see, \eg][]{cesabianchi2006} and the empirical Bernstein inequality \citep{maurer2009}, and Equation~\eqref{apx:event:3} follows by applying Lemma~\ref{lem:bernstein_bernoulli}.

The remaining events concern the summation of Martingale difference sequences:

\begin{align*}
    \evaz\left( \mathcal{F}_{\Vdifftil{}, k, h}, H, L \right) &\coloneqq \bigg\{\sum_{i=1}^k \sum_{j=h}^{H-1} \left[ \sum_{y\in\Ss} P(y|\overline{x},\overline{a})\Vdifftil{i,j+1}(y) - \Vdifftil{i,j+1}(\overline{x}') \right] \\
    &\qquad \le 2 \sqrt{k (H-h) H^2 L} \bigg\},
\end{align*}
\begin{align*}
    \evaz\left( \mathcal{F}'_{\Vdifftil{}, k, h}, \frac{1}{\sqrt{L}}, L \right) &\coloneqq \bigg\{ \sum_{i=1}^k \sum_{j=h}^H \bigg[\sum_{y\in\Ss} P(y|\overline{x}, \overline{a}) \sqrt{\frac{\mathbb{I}(y \in [y]_{i,j})}{N_i(\overline{x}, \overline{a})P(y|\overline{x}, \overline{a})}} \Vdifftil{i,j+1}(y) \\
    &\qquad - \sqrt{\frac{\mathbb{I}(\overline{x}' \in [y]_{i,j})}{N_i(\overline{x}, \overline{a})P(\overline{x}'|\overline{x}, \overline{a})}} \Vdifftil{i,j+1}(\overline{x}') \bigg] \\
    &\qquad \le 2 \sqrt{k(H-h) \frac{1}{\sqrt{L}^2} L}\bigg\},
\end{align*}
\begin{align*}
    \evaz\left( \mathcal{F}_{\Vdifftil{}, k, h, x}, H, L \right) &\coloneqq \bigg\{\sum_{i=1}^k \mathbb{I}(x_{i,h}=x) \sum_{j=h}^{H-1} \left[ \sum_{y\in\Ss} P(y|\overline{x}, \overline{a})\Vdifftil{i,j+1}(y) - \Vdifftil{i,j+1}(\overline{x}') \right] \\
    &\qquad \le 2 \sqrt{N'_{k,h}(x) (H-h) H^2 L} \bigg\},
\end{align*}
\begin{align*}
    \evaz\left( \mathcal{F}'_{\Vdifftil{}, k, h, x}, \frac{1}{\sqrt{L}}, L \right) &\coloneqq \Bigg\{ \sum_{i=1}^k \mathbb{I}(x_{i,h}=x) \sum_{j=h}^H \Bigg[ \bigg[\sum_{y\in\Ss} P(y|\overline{x}, \overline{a}) \sqrt{\frac{\mathbb{I}(y \in [y]_{i,j})}{N_i(\overline{x}, \overline{a})P(y|\overline{x}, \overline{a})}} \Vdifftil{i,j+1}(y) \bigg] \\
    &\qquad - \sqrt{\frac{\mathbb{I}(\overline{x}' \in [y]_{i,j})}{N_i(\overline{x}, \overline{a})P(y|\overline{x}, \overline{a})}} \Vdifftil{i,j+1}(\overline{x}') \Bigg] \\
    &\qquad \le 2 \sqrt{N'_{k,h}(x) (H-h) \frac{1}{\sqrt{L}^2} L}\Bigg\}, \\
\end{align*}
\begin{align*}
    \evfr\left( \mathcal{G}_{\VV, k, h}, H^4T, H^3, L \right) & \coloneqq \Bigg\{ \sum_{i=1}^k \mathbb{E} \left[ \sum_{j=h}^{H-1} \VVpi{j+1} (\overline{x}, \overline{a}) | \mathcal{H}_{i,h} \right] - \sum_{i=1}^k \sum_{j=h}^{H-1} \VVpi{j+1} (\overline{x}, \overline{a}) \Bigg\} \\
    &\qquad \le 2\sqrt{H^4 T_k L} + \frac{4H^3 L}{3},\\
\end{align*}
\begin{align*}
    \evfr\left( \mathcal{G}_{\VV, k, h, x}, H^5 N'_{k,h}(x), H^3, L \right) & \coloneqq \Bigg\{ \sum_{i=1}^k \mathbb{I}(x_{i,h}=x) \mathbb{E} \left[ \sum_{j=h}^{H-1} \VVpi{j+1} (\overline{x}, \overline{a}) | \mathcal{H}_{i,h} \right] \\
    &\qquad - \sum_{i=1}^k \mathbb{I}(x_{i,h}=x) \sum_{j=h}^{H-1} \VVpi{j+1} (\overline{x}, \overline{a}) \Bigg\} \\ 
    &\qquad \le 2\sqrt{H^5 N'_{k,h}(x) L} + \frac{4H^3 L}{3}, \\
\end{align*}
\begin{align*}
    \evaz\left( \mathcal{F}_{\bonus', k, h}, H^2, L \right) &\coloneqq \left\{ \sum_{i=1}^k \sum_{j=h}^{H-1} \left[ \sum_{y \in \Ss} P(y | \overline{x}, \overline{a}) \bonus'_{i,j+1}(y) - \bonus'_{i,j+1}(\overline{x}') \right] \right\} \\
    &\qquad \le 2 \sqrt{(H-h) H^3 T_k L} \\
\end{align*}
\begin{align*}
    \evaz\left( \mathcal{F}_{\bonus', k, h, x}, H^2, L \right) &\coloneqq \left\{ \sum_{i=1}^k \mathbb{I}(x_{i,h}=x) \sum_{j=h}^{H-1} \left[ \sum_{y \in \Ss} P(y | \overline{x}, \overline{a}) \bonus'_{i,j+1}(y) - \bonus'_{i,j+1}(\overline{x}') \right] \right\} \\
    &\qquad \le 2 \sqrt{N'_{k,h}(x) (H-h) H^4 L} \\
\end{align*}
\section{Technical Lemmas}
\label{apx:lemmas}

\begin{lemma}[Bernstein inequality for Bernoulli random variables]
\label{lem:bernstein_bernoulli}
    Let $p$ be the parameter of a Bernoulli random variable, and let $\widehat{p}$ be its estimator. Let $\delta > 0$. Then, \wp at least $1-\delta$, it holds that:
    
    \begin{equation*}
        |\widehat{p} - p| \le \sqrt{\frac{2 p (1-p) L}{n}} + \frac{2 L}{3 n},
    \end{equation*}

    where $n$ represents the number of observations, and $L=\ln(2/\delta)$.
\end{lemma}

\begin{proof}
    Let $\{Y_i\}_{i=1 \ldots, n}$ be the set of i.i.d. realizations of a Bernoulli with parameter $p$. Define the auxiliary random variable:

    \begin{equation*}
        X_i = \frac{Y_i}{n}.
    \end{equation*}

    Observe that $X_1, \ldots, X_n$ are independent random variables, and that $0 \le X_i \le 1/n$. Let $S_n$ be their sum, and $E_n$ be the expected value of $S_n$:

    \begin{align}
        S_n &= \sum_{i=1}^n X_i = \widehat{p}, \nonumber \\
        E_n &= \mathbb{E}[S_n] = \sum_{i=1}^n \mathbb{E}[X_i] = p. \nonumber
    \end{align}

    Let $V_n$ be the variance of $S_n$:

    \begin{equation*}
        V_n = \Var [S_n] = \sum_{i=1}^n \Var [X_i] = \sum_{i=1}^n \left( \frac{1(1-p)}{n^2} \right) = \frac{p(1-p)}{n}.
    \end{equation*}

    By applying Bernstein's inequality, we obtain that:

    \begin{equation}
    \label{lem:bernstein_bernoulli:ineq}
        \mathrm{Pr}(|S_n - E_n| > \epsilon) < 2 \exp\left(- \frac{\epsilon^2 / 2}{V_n + C(\epsilon/3)}\right),
    \end{equation}

    where $C$ is the range of values of the addends in $S_n$ (\ie $C=1/n$). By setting this probability to be equal to $\delta$, we can derive that:

    \begin{equation*}
        \frac{\epsilon^2}{2} = V_n \ln\left(\frac{2}{\delta}\right) + \frac{\epsilon}{3n} \ln \left(\frac{2}{\delta}\right).
    \end{equation*}

    Let $L= \ln(2/\delta)$, by solving the second order polynomial we get that:

    \begin{equation*}
        \epsilon = \frac{L}{3n} \pm \sqrt{\frac{L^2}{9n^2} + 2 V_n L}.
    \end{equation*}

    We can discard the equation with the minus, as it would result in $\epsilon < 0$, thus resulting in the inequality in Equation~\eqref{lem:bernstein_bernoulli:ineq} holding \wp $1 - \delta$. As such, we derive that:

    \begin{align*}
        \epsilon &= \frac{L}{3n} + \sqrt{\frac{L^2}{9n^2} + 2 V_n L} \\
        &\le \frac{L}{3n} + \sqrt{\frac{L^2}{9n^2}} + \sqrt{\frac{2p(1-p)L}{n}} \\
        &= \sqrt{\frac{2p(1-p)L}{n}} + \frac{2L}{3n},
    \end{align*}

    thus completing the proof.
\end{proof}


\begin{lemma}[Regret decomposition upper bound]
\label{lem:regr_dec}
    Let $k \in \dsb{K}$ and $h \in \dsb{H}$. Assume events $\mathcal{E}$ and $\Omega_{k,h}$ hold. Then the regret from stage $h$ onward of all episodes up to $k$ can be upper bounded as follows:

    \begin{align*}
        \sum_{i=1}^k \Vdiff{i,h}(x_{i,h}) \le \sum_{i=1}^k \Vdifftil{i,h}(x_{i,h}) \le & \ e \sum_{i=1}^k \sum_{j=h}^{H-1} \Big[ \varepsilon_{i,j} + 2 \sqrt{L} \overline{\varepsilon}_{i,j} + \bonus_{i,j}(x_{i,j}, a_{i,j}^{\pi_i}) \\
        &\qquad\qquad\quad + \xi_{i,j}(x_{i,j}, a_{i,j}^{\pi_i}) + \frac{8H^2SL}{3 N_i (x_{i,j}, a_{i,j}^{\pi_i})} \Big].
    \end{align*}
\end{lemma}

\begin{proof}
We begin the proof by considering a single value of $k \in \dsb{K}$. Under $\Omega_{k,h}$, we observe that:
\begin{align}
    \Vdiff{k,h}(x_{k,h}) &= \Vstar{h}(x_{k,h}) - \Vpi{h}(x_{k,h}) \nonumber \\
    &\le \Vest{k,h}(x_{k,h}) - \Vpi{h}(x_{k,h}) \nonumber \\
    &= \Vdifftil{k,h}(x_{k,h}). \nonumber
\end{align}

As such, we bound the pseudo-regret $\Vdifftil{k,h}(x_{k,h})$:

\begin{align}
    \Vdifftil{k,h}(x_{k,h}) &= \Vest{k,h}(x_{k,h}) - \Vpi{h}(x_{k,h}) \nonumber \\
    &= \bonus_{k,h}(x_{k,h}, a_{k,h}^{\pi_k}) + \sum_{y \in \Ss} \Pest_k (y | x_{k,h}, a_{k,h}^{\pi_k}) \Vest{k,h+1}(y) - \sum_{y \in \Ss} P(y | x_{k,h}, a_{k,h}^{\pi_k}) \Vpi{h+1}(y) \nonumber \\
    &= \bonus_{k,h}(x_{k,h}, a_{k,h}^{\pi_k}) + \sum_{y \in \Ss} \left[ \Pest_k (y | x_{k,h}, a_{k,h}^{\pi_k}) - P(y | x_{k,h}, a_{k,h}^{\pi_k}) \right] \Vest{k,h+1}(y) \nonumber \\
    &\quad + \sum_{y \in \Ss} P(y | x_{k,h}, a_{k,h}^{\pi_k}) \left[ \Vest{k,h+1}(y) - \Vpi{h+1}(y) \right] \nonumber \\
    &= \bonus_{k,h}(x_{k,h}, a_{k,h}^{\pi_k}) + \sum_{y \in \Ss} \left[ \Pest_k (y | x_{k,h}, a_{k,h}^{\pi_k}) - P(y | x_{k,h}, a_{k,h}^{\pi_k}) \right] \Vstar{h+1} (y) \nonumber \\
    &\quad + \sum_{y \in \Ss} \left[ \Pest_k (y | x_{k,h}, a_{k,h}^{\pi_k}) - P(y | x_{k,h}, a_{k,h}^{\pi_k}) \right] \left[ \Vest{k, h+1}(y) - \Vstar{h+1}(y) \right] \nonumber \\
    &\quad + \sum_{y \in \Ss} P(y | x_{k,h}, a_{k,h}^{\pi_k}) \Vdifftil{k, h+1}(y) \nonumber \\
    &= \Vdifftil{k,h+1}(x_{k,h+1}) + \bonus_{k,h}(x_{k,h}, a_{k,h}^{\pi_k}) + \varepsilon_{k,h} \label{lem:regr_dec:1} \\
    &\quad + \sum_{y \in \Ss} \left[ \Pest_k (y | x_{k,h}, a_{k,h}^{\pi_k}) - P(y | x_{k,h}, a_{k,h}^{\pi_k}) \right] \Vstar{h+1} (y) \nonumber \\
    &\quad + \sum_{y \in \Ss} \left[ \Pest_k (y | x_{k,h}, a_{k,h}^{\pi_k}) - P(y | x_{k,h}, a_{k,h}^{\pi_k}) \right] \left[ \Vest{k, h+1}(y) - \Vstar{h+1}(y) \right] \nonumber \\
    \begin{split} \label{lem:regr_dec:2}
    &\le \Vdifftil{k,h+1}(x_{k,h+1}) + \bonus_{k,h}(x_{k,h}, a_{k,h}^{\pi_k}) + 
    \xi_{k,h}(x_{k,h}, a_{k,h}^{\pi_k}) \\
    &\quad + \varepsilon_{k,h} + \underbrace{\sum_{y \in \Ss} \left[ \Pest_k (y | x_{k,h}, a_{k,h}^{\pi_k}) - P(y | x_{k,h}, a_{k,h}^{\pi_k}) \right] \left[ \Vest{k, h+1}(y) - \Vstar{h+1}(y) \right]}_{(a)},
    \end{split}
\end{align}
where, in Equation~\eqref{lem:regr_dec:1} we apply the definition $\varepsilon_{k,h} \coloneqq P(\cdot | x_{k,h}, a_{k,h}^{\pi_k})^\transpose \Vdifftil{k,h+1}(\cdot) - \Vdifftil{k,h+1}(x_{k,h+1})$, and in Equation~\eqref{lem:regr_dec:2} we apply the definition of $\xi_{k,h}(x_{k,h}, a_{k,h}^{\pi_k})$:

\begin{equation*}
    \xi_{k,h}(x_{k,h}, a_{k,h}^{\pi_k}) = \sum_{y \in \Ss} \left[ \Pest_k (y | x_{k,h}, a_{k,h}^{\pi_k}) - P(y | x_{k,h}, a_{k,h}^{\pi_k}) \right] \Vstar{h+1} (y).
\end{equation*}

Let $\histkh$ be the history of the interactions up to, and including, stage $h$ of episode $k$. Observing that $| \varepsilon_{k,h} | \le H \le + \infty$ and $\mathbb{E}[\varepsilon_{k,h} | \histkh]=0$, we can derive that $\varepsilon_{k,h}$ is a Martingale difference sequence.

We now focus on bounding term $(a)$:

\begin{align}
    (a) &= \sum_{y \in \Ss} \left[ \Pest_k (y | x_{k,h}, a_{k,h}^{\pi_k}) - P(y | x_{k,h}, a_{k,h}^{\pi_k}) \right]\left[ \Vest{k,h+1}(y) - \Vstar{h+1}(y) \right] \nonumber \\
    &\le \sum_{y \in \Ss} \left[ \sqrt{\frac{2 P(y|x_{k,h},a_{k,h}^{\pi_k})(1-P(y|x_{k,h},a_{k,h}^{\pi_k})) L}{N_k(x_{k,h}, a_{k,h}^{\pi_k})}} + \frac{2L}{3 N_k (x_{k,h}, a_{k,h}^{\pi_k})} \right] \Vdifftil{k,h+1}(y) \label{lem:regr_dec:3} \\
    &\le \sum_{y \in \Ss}\sqrt{\frac{2 P(y|x_{k,h},a_{k,h}^{\pi_k})L}{N_k(x_{k,h}, a_{k,h}^{\pi_k})}} \Vdifftil{k,h+1}(y) + \frac{2L}{3N_k(x_{k,h}, a_{k,h}^{\pi_k})} \sum_{y \in \Ss} \Vdifftil{k,h+1}(y) \label{lem:regr_dec:4} \\
    &\le \sqrt{2L} \underbrace{\sum_{y \in \Ss}\sqrt{\frac{P(y|x_{k,h},a_{k,h}^{\pi_k})}{N_k(x_{k,h}, a_{k,h}^{\pi_k})}} \Vdifftil{k,h+1}(y)}_{(b)} + \frac{2SHL}{3N_k(x_{k,h},a_{k,h}^{\pi_k})} \label{lem:regr_dec:5},
\end{align}

where Equation~\eqref{lem:regr_dec:3} is obtained by applying Lemma~\ref{lem:bernstein_bernoulli} to bound $\Pest_k - P$ and by ovserving that $\Vstar{h+1}(y) \ge \Vpi[k]{h+1}(y)$ by definition, Equation~\eqref{lem:regr_dec:4} is obtained by splitting the terms and observing that $1 - P(y|x,a) \le 1$ for every $x,y \in \Ss$ and $a \in \As$, and finally Equation~\eqref{lem:regr_dec:5} is obtained by upper bounding $\Vdifftil{k,h+1}(y)$ with $H$.
To bound term $(b)$, we first need to define the following set of states:

\begin{equation*}
    [y]_{k,h} \coloneqq \{ y \in \Ss : N_k(x_{k,h}, a_{k,h}^{\pi_k}) P(y | x_{k,h}, a_{k,h}^{\pi_k}) \ge 2 H^2L \}.
\end{equation*}

As such, we can rewrite:

\begin{equation}
\label{lem:regr_dec:6}
    (b) = \underbrace{\sum_{y \in [y]_{k,h}}\sqrt{\frac{P(y|x_{k,h},a_{k,h}^{\pi_k})}{N_k(x_{k,h}, a_{k,h}^{\pi_k})}} \Vdifftil{k,h+1}(y)}_{(c)} + \underbrace{\sum_{y \notin [y]_{k,h}}\sqrt{\frac{P(y|x_{k,h},a_{k,h}^{\pi_k})}{N_k(x_{k,h}, a_{k,h}^{\pi_k})}} \Vdifftil{k,h+1}(y)}_{(d)}.
\end{equation}

We now bound term $(c)$ as:

\begin{align}
    (c) &= \sum_{y \in [y]_{k,h}}\sqrt{\frac{P(y|x_{k,h},a_{k,h}^{\pi_k})}{N_k(x_{k,h}, a_{k,h}^{\pi_k})}} \Vdifftil{k,h+1}(y) \nonumber \\
    &= \sum_{y \in [y]_{k,h}} P(y|x_{k,h},a_{k,h}^{\pi_k}) \sqrt{\frac{1}{N_k(x_{k,h},a_{k,h}^{\pi_k})P(y|x_{k,h},a_{k,h}^{\pi_k})}} \Vdifftil{k,h+1}(y) \nonumber \\
    &= \overline{\varepsilon}_{k,h} + \sqrt{\frac{\mathbb{I}(x_{k,h+1}\in[y]_{k,h})}{N_k(x_{k,h},a_{k,h}^{\pi_k})P(x_{k,h+1}|x_{k,h},a_{k,h}^{\pi_k})}}\Vdifftil{k,h+1}(x_{k,h+1}) \label{lem:regr_dec:7} \\
    &\le \overline{\varepsilon}_{k,h} + \sqrt{\frac{1}{2H^2L}}\Vdifftil{k,h+1}(x_{k,h+1}), \label{lem:regr_dec:8}
\end{align}

where Equation~\eqref{lem:regr_dec:7} is obtained by applying the definition of $\overline{\varepsilon}_{k,h}$:

\begin{align*}
    \overline{\varepsilon}_{k,h} &\coloneqq \sum_{y\in\Ss} P(y|x_{k,h},a_{k,h}^{\pi_k}) \sqrt{\frac{\mathbb{I}(y \in [y]_{k,h})}{N_k(x_{k,h},a_{k,h}^{\pi_k})P(y|x_{k,h},a_{k,h}^{\pi_k})}} \Vdifftil{k,h+1}(y) \\
    &\quad - \sqrt{\frac{\mathbb{I}(x_{k,h+1} \in [y]_{k,h})}{N_k(x_{k,h},a_{k,h}^{\pi_k})P(x_{k,h+1}|x_{k,h},a_{k,h}^{\pi_k})}} \Vdifftil{k,h+1}(x_{k,h+1}),
\end{align*}

and Equation~\eqref{lem:regr_dec:8} is obtained by bounding the indicator function with 1, and by applying the definition of $[y]_{k,h}$. With the same reasoning of $\varepsilon_{k,h}$, we can prove that $\overline{\varepsilon}_{k,h}$ is also a Martingale difference sequence.

We can now bound term $(d)$ as follows:

\begin{align}
    (d) &= \sum_{y \notin [y]_{k,h}} \sqrt{\frac{P(y|x_{k,h},a_{k,h}^{\pi_k})}{N_k (x_{k,h},a_{k,h}^{\pi_k})}} \Vdifftil{k,h+1}(y) \nonumber \\
    &= \sum_{y \notin [y]_{k,h}} \sqrt{\frac{N_k (x_{k,h},a_{k,h}^{\pi_k}) P(y|x_{k,h},a_{k,h}^{\pi_k})}{(N_k (x_{k,h},a_{k,h}^{\pi_k}))^2}} \Vdifftil{k,h+1}(y) \nonumber \\
    &\le \frac{H^2S \sqrt{2L}}{N_k (x_{k,h}, a_{k,h}^{\pi_k})}, \label{lem:regr_dec:9}
\end{align}

where Equation~\eqref{lem:regr_dec:9} is obtained by bounding $\Vdifftil{k,h+1}(y)$ with $H$, and by applying the definition of $[y]_{k,h}$. We can now plug the bounds of $(c)$ and $(d)$ into Equation~\eqref{lem:regr_dec:6} to obtain that:

\begin{equation*}
    (b) \le \overline{\varepsilon}_{k,h} + \sqrt{\frac{1}{2H^2L}}\Vdifftil{k,h+1}(x_{k,h+1}) + \frac{H^2S\sqrt{2L}}{N_k(x_{k,h}, a_{k,h}^{\pi_k})}.
\end{equation*}

By plugging the bound of $(b)$ into Equation~\eqref{lem:regr_dec:5}, we obtain that:

\begin{equation*}
    (a) \le \sqrt{2L} \overline{\varepsilon}_{k,h} + \frac{1}{H} \Vdifftil{k,h+1}(x_{k,h+1}) + \frac{8H^2SL}{3N_k(x_{k,h},a_{k,h}^{\pi_k})}.
\end{equation*}

Finally, substituting the bound on $(a)$ into Equation~\eqref{lem:regr_dec:2}, we obtain that:

\begin{align*}
    \Vdifftil{k,h}(x_{k,h}) &\le \left(1 + \frac{1}{H}\right) \Vdifftil{k,h+1}(x_{k,h+1}) + \bonus_{k,h}(x_{k,h}, a_{k,h}^{\pi_k}) + 
    \xi_{k,h} (x_{k,h}, a_{k,h}^{\pi_k})\\
    &\quad + \varepsilon_{k,h} + \sqrt{2L} \overline{\varepsilon}_{k,h} + \frac{8H^2SL}{3N_k(x_{k,h},a_{k,h}^{\pi_k})}.
\end{align*}

We now apply an inductive argument on $\Vdifftil{k,h}(x_{k,h})$ to isolate the term.

Observing that $\Vdifftil{k,H+1}(x_{k,H+1}) = 0$ by definition, we can rewrite:

\begin{align*}
    \Vdifftil{k,h}(x_{k,h}) &\le \sum_{j=h}^{H-1} \gamma_{j-h} \bigg[ \bonus_{k,j}(x_{k,j}, a_{k,j}^{\pi_k}) + 
    \xi_{k,j} (x_{k,j}, a_{k,j}^{\pi_k}) \\
    &\quad + \varepsilon_{k,j} + \sqrt{2L} \overline{\varepsilon}_{k,j} + \frac{8H^2SL}{3N_k(x_{k,j},a_{k,j}^{\pi_k})} \bigg],
\end{align*}

where $\gamma_{j-h} = \left(1+\frac{1}{H}\right)^{j-h}$. Notice that the summation is limited to $H-1$. This will be recurrent throughout the paper and is due to the fact that, the reward being deterministic, there is no uncertainty at $h=H$.
As such, we can assume that the policies $\pi_k$ for $k \in \dsb{K}$ always play greedily at the last stage of each episode.

Observing that $1 + \frac{1}{H} > 1$, we trivially derive that $\gamma_{j-h} \le \gamma_H$ for $j \in \dsb{h,H}$. Recalling that $\lim_{x \to +\infty} \left(1 + \frac{1}{x}\right)^x = e$, we can bound $\gamma_H \le e$, and rewrite:

\begin{align}
    \begin{split}\label{lem:regr_dec:10}
    \Vdifftil{k,h}(x_{k,h}) &\le e \sum_{j=h}^{H-1} \bigg[ \bonus_{k,j}(x_{k,j}, a_{k,j}^{\pi_k}) +
    \xi_{k,j} (x_{k,j}, a_{k,j}^{\pi_k}) \\
    &\quad + \varepsilon_{k,j} + \sqrt{2L} \overline{\varepsilon}_{k,j} + \frac{8H^2SL}{3N_k(x_{k,j},a_{k,j}^{\pi_k})} \bigg],
    \end{split}
\end{align}

To conclude the proof, we need to show that this holds for any value of $k \in \dsb{K}$. Recalling the definition of $\Omega_{k,h}$:

\begin{equation*}
    \Omega_{k,h} \coloneqq \left\{ \Vest{i,j}(x) \ge \Vstar{j}(x), \forall (i,j) \in [k,h]_{\mathrm{hist}}, x \in \Ss \right\},
\end{equation*}

where $[k,h]_{\mathrm{hist}} \coloneqq \left\{ (i,j) : i \in \dsb{K}, j \in \dsb{H}, (i<k) \vee (i=k, j \ge h) \right\}$, we observe that, if $\Omega_{k,h}$ holds, then also the events $\Omega_{i,j}$ hold for $(i,j)\in[k,h]_{hist}$. As such, we can sum up the previous bound of Equation~\eqref{lem:regr_dec:10} over all the episodes $i \in \dsb{k}$, thus concluding the proof.
\end{proof}


\begin{lemma}
\label{lem:sum_eps}
    Let $k \in \dsb{K}$ and $h \in \dsb{H}$. Let events $\mathcal{E}$ and $\Omega_{k,h}$ hold. Then the following bounds hold:
    \begin{align*}
        \sum_{i=1}^k \sum_{j=h}^H \varepsilon_{i,j} &\le 2\sqrt{H^2 T_k L}, \\
        \sum_{i=1}^k \sum_{j=h}^H \overline{\varepsilon}_{i,j} &\le 2 \sqrt{T_k},
    \end{align*}
    where $T_k = kH$.
\end{lemma}

\begin{proof}
Let us first recall the definitions of $\varepsilon_{i,j}$ and $\overline{\varepsilon}_{i,j}$:
\begin{align*}
    \varepsilon_{i,j} &\coloneqq P(\cdot | x_{i,j}, a_{i,j}^{\pi_i})^\transpose \Vdifftil{i,j+1}(\cdot) - \Vdifftil{i,j+1}(x_{i,j+1}), \\
    \overline{\varepsilon}_{i,j} &\coloneqq \sum_{y\in\Ss} P(y|x_{i,j},a_{i,j}^{\pi_i}) \sqrt{\frac{\mathbb{I}(y \in [y]_{i,j})}{N_i(x_{i,j},a_{i,j}^{\pi_i})P(y|x_{i,j},a_{i,j}^{\pi_i})}} \Vdifftil{i,j+1}(y) \\
    &\quad - \sqrt{\frac{\mathbb{I}(x_{i,j+1} \in [y]_{i,j})}{N_i(x_{i,j},a_{i,j}^{\pi_i})P(x_{i,j+1}|x_{i,j},a_{i,j}^{\pi_i})}} \Vdifftil{i,j+1}(x_{i,j+1}),
\end{align*}

where:

\begin{equation*}
    [y]_{k,h} \coloneqq \{ y \in \Ss : N_k(x_{k,h}, a_{k,h}^{\pi_k}) P(y | x_{k,h}, a_{k,h}^{\pi_k}) \ge 2 H^2L \}.
\end{equation*}

Under event $\mathcal{E}$ the following events hold:
\begin{equation*}
    \mathcal{E}_{\mathrm{az}}(\mathcal{F}_{\Vdifftil{}, k, h}, H, L), \quad\text{and}\quad
    \mathcal{E}_{\mathrm{az}}(\mathcal{F}_{\Vdifftil{}, k, h}', 1/\sqrt{L}, L).
\end{equation*}

Event $\mathcal{E}_{\mathrm{az}}(\mathcal{F}_{\Vdifftil{}, k, h}, H, L)$ is defined as the event such that:

\begin{align*}
    \sum_{i=1}^k \sum_{j=h}^{H-1} \left[ \sum_{y\in\Ss} P(y|x_{i,j},a_{i,j}^{\pi_i})\Vdifftil{i,j+1}(y) - \Vdifftil{i,j+1}(x_{i,j+1}) \right] &\le 2 \sqrt{k (H-1-h) H^2 L} \\
    &\le 2 \sqrt{H^2 T_k L}.
\end{align*}

Under this event, we can apply the definition of $\varepsilon_{i,j}$ and derive that:

\begin{equation*}
    \sum_{i=1}^k \sum_{j=h}^{H-1} \varepsilon_{i,j} \le 2 \sqrt{H^2 T_k L}.
\end{equation*}

Event $\mathcal{E}_{\mathrm{az}}(\mathcal{F}_{\Vdifftil{}, k, h}', 1/\sqrt{L}, L)$, on the other hand, is defined as the event such that:

\begin{align*}
    \sum_{i=1}^k \sum_{j=h}^H \bigg[ &\sum_{y\in\Ss} P(y|x_{i,j},a_{i,j}^{\pi_i}) \sqrt{\frac{\mathbb{I}(y \in [y]_{i,j})}{N_i(x_{i,j},a_{i,j}^{\pi_i})P(y|x_{i,j},a_{i,j}^{\pi_i})}} \Vdifftil{i,j+1}(y) \bigg] \\
    &\quad - \sqrt{\frac{\mathbb{I}(y \in [y]_{i,j})}{N_i(x_{i,j},a_{i,j}^{\pi_i})P(y|x_{i,j},a_{i,j}^{\pi_i})}} \Vdifftil{i,j+1}(x_{i,j+1}) \\
    &\le 2 \sqrt{k(H-h) \frac{1}{\sqrt{L}^2} L} \\
    &\le 2 \sqrt{T_k}.
\end{align*}

Under this event, we can apply the definition of $\overline{\varepsilon}_{i,j}$ and derive that:

\begin{equation*}
    \sum_{i=1}^k \sum_{j=h}^{H-1} \overline{\varepsilon}_{i,j} \le 2 \sqrt{T_k},
\end{equation*}

thus concluding the proof.
\end{proof}


\begin{lemma}
\label{lem:sum_eps_x}
    Let $k \in \dsb{K}$, $h \in \dsb{H}$, and $x \in \Ss$. Let events $\mathcal{E}$ and $\Omega_{k,h}$ hold. Then the following bounds hold:
    \begin{align*}
        \sum_{i=1}^k \mathbb{I}(x_{i,h}=x) \sum_{j=h}^H \varepsilon_{i,j} &\le 2\sqrt{H^3 N_{k,h}'(x) L}, \\
        \sum_{i=1}^k \mathbb{I}(x_{i,h}=x) \sum_{j=h}^H \overline{\varepsilon}_{i,j} &\le 2 \sqrt{H N_{k,h}'(x)}.
    \end{align*}
\end{lemma}

\begin{proof}
In a similar way to the proof of Lemma~\ref{lem:sum_eps}, we recall the definitions of $\varepsilon_{i,j}$ and $\overline{\varepsilon}_{i,j}$:
\begin{align*}
    \varepsilon_{i,j} &\coloneqq P(\cdot | x_{i,j}, a_{i,j}^{\pi_i})^\transpose \Vdifftil{i,j+1}(\cdot) - \Vdifftil{i,j+1}(x_{i,j+1}), \\
    \overline{\varepsilon}_{i,j} &\coloneqq \sum_{y\in\Ss} P(y|x_{i,j},a_{i,j}^{\pi_i}) \sqrt{\frac{\mathbb{I}(y \in [y]_{i,j})}{N_i(x_{i,j},a_{i,j}^{\pi_i})P(y|x_{i,j},a_{i,j}^{\pi_i})}} \Vdifftil{i,j+1}(y) \\
    &\quad - \sqrt{\frac{\mathbb{I}(x_{i,j+1} \in [y]_{i,j})}{N_i(x_{i,j},a_{i,j}^{\pi_i})P(x_{i,j+1}|x_{i,j},a_{i,j}^{\pi_i})}} \Vdifftil{i,j+1}(x_{i,j+1}),
\end{align*}

where:

\begin{equation*}
    [y]_{k,h} \coloneqq \{ y \in \Ss : N_k(x_{k,h}, a_{k,h}^{\pi_k}) P(y | x_{k,h}, a_{k,h}^{\pi_k}) \ge 2 H^2L \}.
\end{equation*}

Under event $\mathcal{E}$ the following events hold:
\begin{equation*}
    \mathcal{E}_{\mathrm{az}}(\mathcal{F}_{\Vdifftil{}, k, h, x}, H, L), \quad\text{and}\quad
    \mathcal{E}_{\mathrm{az}}(\mathcal{F}'_{\Vdifftil{}, k, h, x}, 1/\sqrt{L}, L).
\end{equation*}

Event $\mathcal{E}_{\mathrm{az}}(\mathcal{F}_{\Vdifftil{}, k, h, x}, H, L)$ is defined as the event such that:

\begin{align*}
    \sum_{i=1}^k \mathbb{I}(x_{i,h}=x) \sum_{j=h}^{H-1} \left[ \sum_{y\in\Ss} P(y|x_{i,j},a_{i,j}^{\pi_i})\Vdifftil{i,j+1}(y) - \Vdifftil{i,j+1}(x_{i,j+1}) \right] \le 2 \sqrt{H^3 N'_{k,h}(x) L}.
\end{align*}

Under this event, we can apply the definition of $\varepsilon_{i,j}$ and derive that:

\begin{equation*}
    \sum_{i=1}^k \mathbb{I}(x_{i,h}=x) \sum_{j=h}^{H-1} \varepsilon_{i,j} \le 2 \sqrt{H^3 N'_{k,h}(x) L}.
\end{equation*}

Event $\mathcal{E}_{\mathrm{az}}(\mathcal{F}'_{\Vdifftil{}, k, h, x}, 1/\sqrt{L}, L)$, on the other hand, is defined as the event such that:

\begin{align*}
    \sum_{i=1}^k \mathbb{I}(x_{i,h}=x) \sum_{j=h}^H \bigg[ &\sum_{y\in\Ss} P(y|x_{i,j},a_{i,j}^{\pi_i}) \sqrt{\frac{\mathbb{I}(y \in [y]_{i,j})}{N_i(x_{i,j},a_{i,j}^{\pi_i})P(y|x_{i,j},a_{i,j}^{\pi_i})}} \Vdifftil{i,j+1}(y) \bigg] \\
    &\quad - \sqrt{\frac{\mathbb{I}(x_{i,j+1} \in [y]_{i,j})}{N_i(x_{i,j},a_{i,j}^{\pi_i})P(x_{i,j+1}|x_{i,j},a_{i,j}^{\pi_i})}} \Vdifftil{i,j+1}(x_{i,j+1}) \\
    &\le 2 \sqrt{N'_{k,h}(x) (H-h) \frac{1}{\sqrt{L}^2} L} \\
    &\le 2 \sqrt{H N'_{k,h}(x)}.
\end{align*}

Under this event, we can apply the definition of $\overline{\varepsilon}_{i,j}$ and derive that:

\begin{equation*}
    \sum_{i=1}^k \mathbb{I}(x_{i,h}=x) \sum_{j=h}^{H-1} \overline{\varepsilon}_{i,j} \le 2 \sqrt{H N'_{k,h}(x)},
\end{equation*}

thus concluding the proof.

\end{proof}


\begin{lemma}
\label{lem:sum_next_state_var}
Let $k \in \dsb{K}$ and $h \in \dsb{H}$. Let $\pi_k$ be the policy followed during episode $k$. Under the events $\mathcal{E}$ and $\Omega_{k,h}$, the following holds for every $x \in \Ss$:

\begin{align*}
    \sum_{i=1}^k \sum_{j=h}^{H-1} \VVpi[i]{j+1} (x_{i,j}, a_{i,j}^{\pi_i}) &\le HT_k + 2\sqrt{H^4 T_k L} + \frac{4}{3}H^3L, \\
    \sum_{i=1}^k \mathbb{I}(x_{i,h}=x) \sum_{j=h}^{H-1} \VVpi[i]{j+1} (x_{i,j}, a_{i,j}^{\pi_i}) &\le H^2 N'_{k,h}(x) + 2 \sqrt{H^5 N'_{k,h}(x) L} + \frac{4}{3} H^3 L.
\end{align*}
\end{lemma}

\begin{proof}
We begin the proof by restating the definition of $\VVpi{j+1} (x_{i,j}, a_{i,j}^{\pi_i})$:

\begin{equation*}
    \VVpi{j+1} (x_{i,j}, a_{i,j}^{\pi_i}) \coloneqq \Var_{y \sim P(\cdot | x_{i,j}, a_{i,j}^{\pi_i})}[\Vpi{j+1} (y)]
\end{equation*}

Under event $\mathcal{E}$, the following events hold:

\begin{equation*}
    \mathcal{E}_{\mathrm{fr}} (\mathcal{G}_{\VV, k, h}, H^4 T_k, H^3, L) \quad\text{and}\quad \mathcal{E}_{\mathrm{fr}} (\mathcal{G}_{\VV, k, h, x}, H^5 N'_{k,h}, H^3, L).
\end{equation*}

Event $\mathcal{E}_{\mathrm{fr}} (\mathcal{G}_{\VV, k, h}, H^4 T_k, H^3, L)$ is defined as the event such that:

\begin{equation}
    \sum_{i=1}^k \sum_{j=h}^{H-1} \VVpi{j+1} (x_{i,j}, a_{i,j}^{\pi_i}) - \sum_{i=1}^k \mathbb{E} \left[ \sum_{j=h}^{H-1} \VVpi{j+1} (x_{i,j}, a_{i,j}^{\pi_i}) | \histkh \right] \le 2\sqrt{H^4 T_k L} + \frac{4H^3 L}{3}, \nonumber
\end{equation}

which implies that:

\begin{equation}
    \sum_{i=1}^k \sum_{j=h}^{H-1} \VVpi{j+1} (x_{i,j}, a_{i,j}^{\pi_i}) \le \sum_{i=1}^k \mathbb{E} \left[ \sum_{j=h}^{H-1} \VVpi{j+1} (x_{i,j}, a_{i,j}^{\pi_i}) | \histkh \right] + 2\sqrt{H^4 T_k L} + \frac{4H^3 L}{3}. \label{lem:sum_next_state_var:1}
\end{equation}

On the other hand, event $\mathcal{E}_{\mathrm{fr}} (\mathcal{G}_{\VV, k, h, x}, H^5 N'_{k,h}, H^3, L)$ is defined as the event such that:

\begin{align}
    \sum_{i=1}^k \mathbb{I}(x_{i,h}=x) \sum_{j=h}^{H-1} \VVpi{j+1} (x_{i,j}, a_{i,j}^{\pi_i}) &- \sum_{i=1}^k \mathbb{I}(x_{i,h}=x) \mathbb{E} \left[ \sum_{j=h}^{H-1} \VVpi{j+1} (x_{i,j}, a_{i,j}^{\pi_i}) | \histkh \right] \nonumber \\ 
    &\le 2\sqrt{H^5 N'_{k,h}(x) L} + \frac{4H^3 L}{3}, \nonumber
\end{align}

which implies that:

\begin{align}
    \begin{split}
    \label{lem:sum_next_state_var:2}
    \sum_{i=1}^k \mathbb{I}(x_{i,h}=x) \sum_{j=h}^{H-1} \VVpi{j+1} (x_{i,j}, a_{i,j}^{\pi_i}) &\le \sum_{i=1}^k \mathbb{I}(x_{i,h}=x) \mathbb{E} \left[ \sum_{j=h}^{H-1} \VVpi{j+1} (x_{i,j}, a_{i,j}^{\pi_i}) | \histkh \right] \\ 
    &\quad + 2\sqrt{H^5 N'_{k,h}(x) L} + \frac{4H^3 L}{3}. 
    \end{split}
\end{align}

Observe that by applying the \emph{law of total variance} \citep[LTV, see, \eg Theorem 9.5.5 of][]{blitzstein2019}, we can write:

\begin{align}
    \begin{split}
    \label{lem:sum_next_state_var:3}
    \underset{x_{i,h+1}, \ldots, x_{i, H-1}}{\Var} \left[ \sum_{j=h}^{H-1} R^\pi (x_{i,j}) \right] &= \underbrace{\underset{x_{i,h+1}}{\Var}\left[ \underset{x_{i,h+2}, \ldots, x_{i,H-1}}{\mathbb{E}} \left[ \sum_{j=h}^{H-1} R^\pi (x_{i,j}) \bigg| x_{i,h+1} \right] \right]}_{(a)} \\
    &\quad + \underset{x_{i,h+1}}{\mathbb{E}} \left[ \underbrace{\underset{x_{i,h+2}, \ldots, x_{i,H-1}}{\Var} \left[ \sum_{j=h}^{H-1} R^\pi (x_{i,j}) \bigg| x_{i,h+1} \right]}_{(b)} \right].
    \end{split}
\end{align}

Term $(a)$ can be rewritten as:

\begin{align}
    (a) &= \underset{x_{i,h+1}}{\Var} \left[ R^\pi (x_{i,h}) + \underset{x_{i,h+2}, \ldots, x_{i,H-1}}{\mathbb{E}} \left[ \sum_{j=h+1}^{H-1} R^\pi (x_{i,j}) \bigg| x_{i,h+1} \right] \right] \nonumber \\
    &= \underset{x_{i,h+1}}{\Var} \left[ \Vpi{h+1} (x_{i, h+1}) \right] \label{lem:sum_next_state_var:4} \\
    &= \VVpi{h+1}(x_{i,h}, a_{i,h}^{\pi_i}),
\end{align}

where Equation~\eqref{lem:sum_next_state_var:4} is obtained by observing that $R^\pi (x_{i,h})$ has zero variance \wrt $x_{i,h+1}$, and by applying the definition of value function.

We can then recursively apply the LTV to term $(b)$ and, considering the expectation over the trajectory generated following policy $\pi$ from stage $h$ onward, we can write:

\begin{equation}
\label{lem:sum_next_state_var:5}
    \underset{x_{i,h+1}, \ldots, x_{i, H-1}}{\Var} \left[ \sum_{j=h}^{H-1} R^\pi (x_{i,j}) \right] = \mathbb{E} \left[ \sum_{j=h}^{H-1} \VVpi{j+1} (x_{i,j}, a_{i,j}^{\pi_i}) \right].
\end{equation}

By applying the result of Equation~\eqref{lem:sum_next_state_var:5} to Equations~\eqref{lem:sum_next_state_var:1} and \eqref{lem:sum_next_state_var:2}, we get:

\begin{align}
    \sum_{i=1}^k \mathbb{E} \left[ \sum_{j=h}^{H-1} \VVpi{j+1} (x_{i,j}, a_{i,j}^{\pi_i}) | \histkh \right] &= \sum_{i=1}^k \Var \left[ \sum_{j=h+1}^{H-1} R^\pi (x_{i,j}) \right] \nonumber \\
    &\le k (H-h)^2 \nonumber \\
    &\le H T_k, \label{lem:sum_next_state_var:6}
\end{align}

and:

\begin{align}
    \sum_{i=1}^k \mathbb{I}(x_{i,h}=x) \mathbb{E} \left[ \sum_{j=h}^{H-1} \VVpi{j+1} (x_{i,j}, a_{i,j}^{\pi_i}) | \histkh \right] &= \sum_{i=1}^k \mathbb{I}(x_{i,h}=x) \Var\left[ \sum_{j=h+1}^{H-1} R^\pi (x_{i,j}) \right] \nonumber \\
    &\le N'_{k,h}(x) (H-h)^2 \nonumber \\
    &\le H^2 N'_{k,h}(x). \label{lem:sum_next_state_var:7}
\end{align}

Finally, we can plug Equations~\eqref{lem:sum_next_state_var:6} and \eqref{lem:sum_next_state_var:7} into Equations~\eqref{lem:sum_next_state_var:1} and \eqref{lem:sum_next_state_var:2}, respectively, obtaining:

\begin{align*}
    \sum_{i=1}^k \sum_{j=h}^{H-1} \VVpi{j+1} (x_{i,j}, a_{i,j}^{\pi_i}) &\le HT_k + 2\sqrt{H^4 T_k L} + \frac{4}{3}H^3L, \\
    \sum_{i=1}^k \mathbb{I}(x_{i,h}=x) \sum_{j=h}^{H-1} \VVpi{j+1} (x_{i,j}, a_{i,j}^{\pi_i}) &\le H^2 N'_{k,h}(x) + 2 \sqrt{H^5 N'_{k,h}(x) L} + \frac{4}{3} H^3 L,
\end{align*}

thus concluding the proof.
\end{proof}


\begin{lemma}
\label{lem:sum_var_diff_opt}
    Let $k \in \dsb{K}$ and $h \in \dsb{H}$. Let $\pi_k$ be the policy played during episode $k$. Under the events $\mathcal{E}$ and $\Omega_{k,h}$, the following holds for every $x \in \Ss$:
    \begin{align*}
        \sum_{i=1}^k \sum_{j=h}^{H-1} \left( \VVopt{j+1} (x_{i,j}, a_{i,j}^{\pi_i}) - \VVpi[i]{j+1} (x_{i,j}, a_{i,j}^{\pi_i}) \right) &\le 2 H \sum_{i=1}^k \sum_{j=h}^{H-1} \Vdifftil{i,j}(x_{i,j}) + 4 H^2 \sqrt{T_k L}, \\
        \sum_{i=1}^k \mathbb{I}(x_{i,h}=x) \sum_{j=h}^{H-1} \left( \VVopt{j+1} (x_{i,j}, a_{i,j}^{\pi_i}) - \VVpi[i]{j+1} (x_{i,j}, a_{i,j}^{\pi_i}) \right) &\le 2 H \sum_{i=1}^k \mathbb{I}(x_{i,h}=x) \sum_{j=h}^{H-1} \Vdifftil{i,h}(x_{i,h}) \\
        &\quad + 4H^2 \sqrt{H N'_{k,h}(x) L}.
    \end{align*}
\end{lemma}

\begin{proof}
We demonstrate the result by providing an upper bound to $\VVopt{j+1} - \VVpi[i]{j+1}$ first, and then bounding its summation over episodes and stages.
We can demonstrate that:

\begin{align}
    \VVopt{j+1} (x_{i,j}, a_{i,j}^{\pi_i}) - \VVpi[i]{j+1} (x_{i,j}, a_{i,j}^{\pi_i}) &= \Var_{y \sim P(\cdot | x_{i,j}, a_{i,j}^{\pi_i})} [\Vstar{j+1}(y)] - \Var_{y \sim P(\cdot | x_{i,j}, a_{i,j}^{\pi_i})} [\Vpi[i]{j+1} (y)] \nonumber \\
    &\le \mathbb{E}_{y \sim P(\cdot | x_{i,j}, a_{i,j}^{\pi_i})} [(\Vstar{j+1}(y))^2 - (\Vpi[i]{j+1}(y))^2] \label{lem:sum_var_diff_opt:1} \\
    &\le 2H \mathbb{E}_{y \sim P(\cdot | x_{i,j}, a_{i,j}^{\pi_i})} [\Vstar{j+1}(y) - \Vpi[i]{j+1}(y)] \label{lem:sum_var_diff_opt:2},
\end{align}

where Equation~\eqref{lem:sum_var_diff_opt:1} is obtained by applying the definition of variance and observing that $\Vstar{j+1}(x) \ge \Vpi[i]{j+1}(x)$ by definition, and Equation~\eqref{lem:sum_var_diff_opt:2} is obtained by expanding the square and by observing that $\Vpi[]{j+1}(x) \le \Vstar{j+1}(x) \le H$.

Using the argument of Equation~\eqref{lem:sum_var_diff_opt:2}, we obtain the following inequalities:

\begin{align}
    \begin{split}
    \label{lem:sum_var_diff_opt:3}
    \sum_{i=1}^k \sum_{j=h}^{H-1} \big( \VVopt{j+1} (x_{i,j}, a_{i,j}^{\pi_i}) &- \VVpi[i]{j+1} (x_{i,j}, a_{i,j}^{\pi_i}) \big) \\
    &\le 2 H \underbrace{\sum_{i=1}^k \sum_{j=h}^{H-1} \mathbb{E}_{y \sim P(\cdot | x_{i,j}, a_{i,j}^{\pi_i})} [\Vdiff{i,j+1}(y)]}_{(a)},
    \end{split} \\
    \begin{split}
    \label{lem:sum_var_diff_opt:4}
    \sum_{i=1}^k \mathbb{I}(x_{i,h}=x) \sum_{j=h}^{H-1} \big( \VVopt{j+1} (x_{i,j}, a_{i,j}^{\pi_i}) &- \VVpi[i]{j+1} (x_{i,j}, a_{i,j}^{\pi_i}) \big) \\
    &\le 2 H \underbrace{\sum_{i=1}^k \mathbb{I}(x_{i,h}=x) \sum_{j=h}^{H-1} \mathbb{E}_{y \sim P(\cdot | x_{i,j}, a_{i,j}^{\pi_i})} [\Vdiff{i,j+1}(y)]}_{(b)}.
    \end{split}
\end{align}

We now bound term $(a)$ as follows:

\begin{align}
    (a) &\le \sum_{i=1}^k \sum_{j=h}^{H-1} \mathbb{E}_{y \sim P(\cdot | x_{i,j}, a_{i,j}^{\pi_i})} [\Vdifftil{i, j+1}(y)] \label{lem:sum_var_diff_opt:5} \\
    &\le 2 \sqrt{H^2 T_k L} + \sum_{i=1}^k \sum_{j=h}^{H-1} \Vdifftil{i,j+1}(x_{i,j+1}) \label{lem:sum_var_diff_opt:6}
\end{align}

where Equation~\eqref{lem:sum_var_diff_opt:5} is obtained because, under $\Omega_{k,h}$, it holds that ${\Vstar{j+1}(y) \le \Vest{i,j+1}(y)}$. Equation~\eqref{lem:sum_var_diff_opt:6} is obtained by considering that, under event $\mathcal{E}$, the event $\mathcal{E}_{\mathrm{az}} (\mathcal{F}_{\Vdifftil{}, k, h}, H, L)$ holds, as shown in Lemma~\ref{lem:sum_eps}.

Following a similar procedure, we bound term $(b)$ by considering event $\mathcal{E}_{\mathrm{az}}(\mathcal{F}_{\Vdifftil{}, k, h, x}, H, L)$, thus obtaining:

\begin{equation}
\label{lem:sum_var_diff_opt:7}
    (b) \le 2H \sqrt{H N'_{k,h}(x) L} + \sum_{i=1}^k \mathbb{I}(x_{i,h}=x) \sum_{j=h}^{H-1} \Vdifftil{i,j+1}(x_{i,j+1}).
\end{equation}

We can then plug Equations~\eqref{lem:sum_var_diff_opt:6} and \eqref{lem:sum_var_diff_opt:7} into Equations~\eqref{lem:sum_var_diff_opt:3} and \eqref{lem:sum_var_diff_opt:4}, respectively, to write:

\begin{align*}
    \sum_{i=1}^k \sum_{j=h}^{H-1} \left( \VVopt{j+1} (x_{i,j}, a_{i,j}^{\pi_i}) - \VVpi[i]{j+1} (x_{i,j}, a_{i,j}^{\pi_i}) \right) &\le 2 H \sum_{i=1}^k \sum_{j=h}^{H-1} \Vdifftil{i,j}(x_{i,j}) + 4 H^2 \sqrt{T_k L}, \\
    \sum_{i=1}^k \mathbb{I}(x_{i,h}=x) \sum_{j=h}^{H-1} \left( \VVopt{j+1} (x_{i,j}, a_{i,j}^{\pi_i}) - \VVpi[i]{j+1} (x_{i,j}, a_{i,j}^{\pi_i}) \right) &\le 2 H \sum_{i=1}^k \mathbb{I}(x_{i,h}=x) \sum_{j=h}^{H-1} \Vdifftil{i,j}(x_{i,j}) \\
    &\quad + 4H^2 \sqrt{H N'_{k,h}(x) L},
\end{align*}

thus concluding the proof.
\end{proof}


\begin{lemma}
\label{lem:sum_var_diff_est}
Let $k \in \dsb{K}$ and $h \in \dsb{H}$. Let $\pi_k$ denote the policy followed during episode $k$. Under events $\mathcal{E}$ and $\Omega_{k,h}$, the following inequalities hold for every $x \in \Ss$:

\begin{align*}
    \sum_{i=1}^k \sum_{j=h}^{H-1} \VVest{i,j+1}(x_{i,j}, a_{i,j}^{\pi_i}) &- \VVpi[i]{j+1}(x_{i,j}, a_{i,j}^{\pi_i}) \\
    &\le 7 H^2 S \sqrt{A T_k L} + 2H \sum_{i=1}^k \sum_{j=h}^{H-1} \Vdifftil{i,j+1}(x_{i,j+1}),\\
    \sum_{i=1}^k \mathbb{I}(x_{i,j}=x) \sum_{j=h}^{H-1} \VVest{i,j+1}(&x_{i,j}, a_{i,j}^{\pi_i}) - \VVpi[i]{j+1}(x_{i,j}, a_{i,j}^{\pi_i}) \\
    &\le 7 H^2 S \sqrt{H A N'_{k,h}(x) L} + 2H \sum_{i=1}^k \mathbb{I}(x_{i,h}=x) \sum_{j=h}^{H-1} \Vdifftil{i,j+1}(x_{i,j+1}).
\end{align*}
\end{lemma}

\begin{proof}
Similarly to the proof of Lemma~\ref{lem:sum_var_diff_opt}, we demonstrate the result by providing an upper bound to $\VVest{i,j+1} - \VVpi[i]{j+1}$ first, and then bounding its summation over episodes and stages.

\begin{align}
    \VVest{i,j+1}(x_{i,j}, a_{i,j}^{\pi_i}) &- \VVpi[i]{j+1}(x_{i,j}, a_{i,j}^{\pi_i}) \nonumber \\
    &= \Var_{y \sim \Pest_i (\cdot | x_{i,j}, a_{i,j}^{\pi_i})}[\Vest{i,j+1}(y)] - \Var_{y \sim P (\cdot | x_{i,j}, a_{i,j}^{\pi_i})} [\Vpi[i]{j+1} (y)] \nonumber \\
    &= \mathbb{E}_{y \sim \Pest_i (\cdot | x_{i,j}, a_{i,j}^{\pi_i})} [(\Vest{i,j+1}(y))^2] - \mathbb{E}_{y \sim \Pest_i (\cdot | x_{i,j}, a_{i,j}^{\pi_i})} [\Vest{i,j+1}(y)]^2 \nonumber \\
    &\quad - \mathbb{E}_{y \sim P (\cdot | x_{i,j}, a_{i,j}^{\pi_i})} [(\Vpi[i]{j+1}(y))^2] + \mathbb{E}_{y \sim P (\cdot | x_{i,j}, a_{i,j}^{\pi_i})} [\Vpi[i]{j+1}(y)]^2 \nonumber \\
    \begin{split}
    \label{lem:sum_var_diff_est:1}
    &\le \mathbb{E}_{y \sim \Pest_i (\cdot | x_{i,j}, a_{i,j}^{\pi_i})} [(\Vest{i,j+1}(y))^2] - \mathbb{E}_{y \sim P (\cdot | x_{i,j}, a_{i,j}^{\pi_i})} [(\Vpi[i]{j+1}(y))^2] \\
    &\quad + \mathbb{E}_{y \sim P (\cdot | x_{i,j}, a_{i,j}^{\pi_i})} [\Vstar{j+1}(y)]^2 - \mathbb{E}_{y \sim \Pest_i (\cdot | x_{i,j}, a_{i,j}^{\pi_i})} [\Vstar{j+1}(y)]^2
    \end{split} \\
    \begin{split}
    \label{lem:sum_var_diff_est:2}
    &\le \mathbb{E}_{y \sim \Pest_i (\cdot | x_{i,j}, a_{i,j}^{\pi_i})} [(\Vest{i,j+1}(y))^2] - \mathbb{E}_{y \sim P (\cdot | x_{i,j}, a_{i,j}^{\pi_i})} [(\Vest{i,j+1}(y))^2] \\
    &\quad + \mathbb{E}_{y \sim P (\cdot | x_{i,j}, a_{i,j}^{\pi_i})} [(\Vest{i,j+1}(y))^2] - \mathbb{E}_{y \sim P (\cdot | x_{i,j}, a_{i,j}^{\pi_i})} [(\Vpi[i]{j+1}(y))^2] \\
    &\quad + 2H \sum_{y \in \Ss}(P(y|x_{i,j}, a_{i,j}^{\pi_i}) - \Pest_i (y|x_{i,j}, a_{i,j}^{\pi_i})) \Vstar{j+1}(y)
    \end{split} \\
    \begin{split}
    \label{lem:sum_var_diff_est:3}
    &\le \mathbb{E}_{y \sim \Pest_i (\cdot | x_{i,j}, a_{i,j}^{\pi_i})} [(\Vest{i,j+1}(y))^2] - \mathbb{E}_{y \sim P (\cdot | x_{i,j}, a_{i,j}^{\pi_i})} [(\Vest{i,j+1}(y))^2] \\
    &\quad + \mathbb{E}_{y \sim P (\cdot | x_{i,j}, a_{i,j}^{\pi_i})} [(\Vest{i,j+1}(y))^2] - \mathbb{E}_{y \sim P (\cdot | x_{i,j}, a_{i,j}^{\pi_i})} [(\Vpi[i]{j+1}(y))^2] \\
    &\quad + 4 H \sqrt{\frac{H^2 L}{N_i (x_{i,j}, a_{i,j}^{\pi_i})}} 
    \end{split}
\end{align}

where Equation~\eqref{lem:sum_var_diff_est:1} follows from the fact that, under $\Omega_{k,h}$, $\Vest{i,j}(y) \ge \Vstar{j}(y) \ge \Vpi[i]{j}(y)$. Equation~\eqref{lem:sum_var_diff_est:2} is obtained by adding and subtracting $\mathbb{E}_{y \sim P (\cdot | x_{i,j}, a_{i,j}^{\pi_i})} [(\Vest{i,j+1}(y))^2]$, and by observing that $\Vstar{j}(y) \le H$. Equation~\eqref{lem:sum_var_diff_est:3} is obtained by bounding the model error via Hoeffding's inequality.

Putting this result into the double summation, we get:
\begin{align}
\begin{split}
\label{lem:sum_var_diff_est:4}
    \sum_{i=1}^k \sum_{j=h}^{H-1} &\VVest{i,j+1}(x_{i,j}, a_{i,j}^{\pi_i}) - \VVpi[i]{j+1}(x_{i,j}, a_{i,j}^{\pi_i}) \\
    &\le \underbrace{\sum_{i=1}^k \sum_{j=h}^{H-1} \left[ \mathbb{E}_{y \sim \Pest_i (\cdot | x_{i,j}, a_{i,j}^{\pi_i})} [(\Vest{i,j+1}(y))^2] - \mathbb{E}_{y \sim P (\cdot | x_{i,j}, a_{i,j}^{\pi_i})} [(\Vest{i,j+1}(y))^2] \right]}_{(a)} \\
    &\quad + \underbrace{\sum_{i=1}^k \sum_{j=h}^{H-1} \mathbb{E}_{y \sim P (\cdot | x_{i,j}, a_{i,j}^{\pi_i})} [(\Vest{i,j+1}(y))^2] - \mathbb{E}_{y \sim P (\cdot | x_{i,j}, a_{i,j}^{\pi_i})} [(\Vpi[i]{j+1}(y))^2] }_{(b)} \\
    &\quad + \underbrace{\sum_{i=1}^k \sum_{j=h}^{H-1} 4 H \sqrt{\frac{H^2 L}{N_i (x_{i,j}, a_{i,j}^{\pi_i})}}  }_{(c)}.
\end{split}
\end{align}

We begin by bounding term $(a)$:

\begin{align}
    (a) &\le \sum_{i=1}^k \sum_{j=h}^{H-1} H^2 \| \Pest_i (\cdot | x_{i,j}, a_{i,j}^{\pi_i}) - P ( \cdot | x_{i,j}, a_{i,j}^{\pi_i}) \|_1 \nonumber \\
    &\le \sum_{i=1}^k \sum_{j=h}^{H-1} 2H^2 \sqrt{\frac{SL}{N_i (x_{i,j}, a_{i,j}^{\pi_i})}} \label{lem:sum_var_diff_est:5} \\
    &= 2 H^2 \sqrt{SL} \sum_{x\in\Ss} \sum_{a \in \As} \sum_{n=1}^{N_i(x,a)} n^{-1/2} \nonumber \\
    &\le 2 H^2 \sqrt{SL} \sum_{x\in\Ss} \sum_{a \in \As} \sum_{n=1}^{\frac{kH}{SA}} n^{-1/2} \nonumber \\
    &\le H^2 S \sqrt{A T_k L}, \nonumber
\end{align}

where Equation~\eqref{lem:sum_var_diff_est:5} follows by applying the result of Theorem~{2.1} of \cite{weissman2003}, which holds under event $\mathcal{E}$.

We now bound term $(b)$:
\begin{align}
    (b) &= \sum_{i=1}^k \sum_{j=h}^{H-1} \mathbb{E}_{y \sim P (\cdot | x_{i,j}, a_{i,j}^{\pi_i})} [(\Vest{i,j+1}(y) + \Vpi{j+1}(y))(\Vest{i,j+1}(y) - \Vpi{j+1}(y))] \nonumber \\
    &\le 2H \sum_{i=1}^k \sum_{j=h}^{H-1} \mathbb{E}_{y \sim P (\cdot | x_{i,j}, a_{i,j}^{\pi_i})} [\Vdifftil{i,j+1}(y)] \nonumber \\
    &\le 2H(\sum_{i=1}^k \sum_{j=h}^{H-1} \Vdifftil{i,j+1}(x_{i,j+1}) + 2H \sqrt{T_k L}), \label{lem:sum_var_diff_est:6}
\end{align}

where Equation~\eqref{lem:sum_var_diff_est:6} is obtained by observing that, under event $\mathcal{E}$, event $\mathcal{E}_{\mathrm{az}}(\mathcal{F}_{\Vdifftil{}, k, h}, H, L)$ holds.
We now bound term $(c)$:

\begin{align*}
    (c) &\le 4 H^2 \sqrt{L} \sum_{x\in\Ss} \sum_{a\in\As} \sum_{n=1}^{\frac{kH}{SA}} n^{-1/2} \\
    &\le 2 H^2 \sqrt{S A T_k L}.
\end{align*}

Finally, by plugging the bounds of terms $(a)$, $(b)$, and $(c)$ into Equation~\eqref{lem:sum_var_diff_est:4}, we get:

\begin{align*}
    \sum_{i=1}^k \sum_{j=h}^{H-1} \VVest{i,j+1}(x_{i,j}, a_{i,j}^{\pi_i}) &- \VVpi[i]{j+1}(x_{i,j}, a_{i,j}^{\pi_i}) \\
    &\le H^2 S \sqrt{A T_k L} + 2H \sum_{i=1}^k \sum_{j=h}^{H-1} \Vdifftil{i,j+1}(x_{i,j+1}) \\
    &\quad + 4H^2 \sqrt{T_k L} + 2H^2 \sqrt{SAT_kL} \\
    &\le 7 H^2 S \sqrt{A T_k L} + 2H \sum_{i=1}^k \sum_{j=h}^{H-1} \Vdifftil{i,j+1}(x_{i,j+1}).
\end{align*}

Using the same procedure, we can bound the following summation as:

\begin{align*}
    \sum_{i=1}^k \mathbb{I}(x_{i,h}=x) \sum_{j=h}^{H-1} &\VVest{i,j+1}(x_{i,j}, a_{i,j}^{\pi_i}) - \VVpi[i]{j+1}(x_{i,j}, a_{i,j}^{\pi_i}) \\
    &\le 7 H^2 S \sqrt{H A N'_{k,h}(x) L} + 2H \sum_{i=1}^k \mathbb{I}(x_{i,h}=x) \sum_{j=h}^{H-1} \Vdifftil{i,j+1}(x_{i,j+1}),
\end{align*}

thus concluding the proof.
\end{proof}


\begin{lemma}[Summation over typical episodes of state-action wise model errors]
\label{lem:sum_c}
Let $k \in \dsb{K}$ and $h \in \dsb{H}$. Let $\pi_k$ be the policy followed during episode $k$. Under events $\mathcal{E}$ and $\Omega_{k,h}$, the following inequalities hold for every $x \in \Ss$:

\begin{align}
    \begin{split}
    \label{lem:sum_c:1}
        \sum_{i=1}^k \mathbb{I}(i \in \ktyp) &\sum_{j=h}^{H-1} \left[ \Pest_i(\cdot | x_{i,j}, a_{i,j}^{\pi_i}) - P(\cdot | x_{i,j}, a_{i,j}^{\pi_i}) \right]^\transpose \Vstar{j+1}(\cdot) \\
        &\le \sqrt{6 H S A T_k L^2} + \frac{2}{3} HSAL^2 \\
        &\quad + 2\sqrt{H S A L^2 \sum_{i=1}^k \sum_{j=h}^{H-1} \Vdifftil{i,j}(x_{i,j})},
    \end{split} \\
    \begin{split}
    \label{lem:sum_c:2}
        \sum_{i=1}^k \mathbb{I}(i \in \ktypx, x_{i,h}=x) &\sum_{j=h}^{H-1} \left[ \Pest_i(\cdot | x_{i,j}, a_{i,j}^{\pi_i}) - P(\cdot | x_{i,j}, a_{i,j}^{\pi_i}) \right]^\transpose \Vstar{j+1}(\cdot) \\
        &\le \sqrt{6 H^2 S A N'_{k,h}(x) L^2} + \frac{2}{3} HSAL^2 \\
        &\quad + 2\sqrt{H S A L^2 \sum_{i=1}^k \mathbb{I}(x_{i,h}=x) \sum_{j=h}^{H-1} \Vdifftil{i,j}(x_{i,j})},
    \end{split}
\end{align}

where:

\begin{align*}
    \ktyp &\coloneqq \{ i\in\dsb{k} : (x_{i,h}, a_{i,h}^{\pi_i}) \in [(x,a)]_k, i \ge 250 H S^2 A L, \forall h \in \dsb{H} \}, \\
    \ktypx &\coloneqq \{ i \in \dsb{k} : (x_{i,h}, a_{i,h}^{\pi_i}) \in [(x,a)]_k, N'_{k,h}(x) \ge 250 H S^2 A L, \forall h \in \dsb{H} \}, \\
    [(x,a)]_k &\coloneqq \{ (x,a) \in \Ss \times \As : N_k (x,a) \ge H, N'_{k,h}(x) \ge H, \forall h \in \dsb{H} \}.
\end{align*}
\end{lemma}

\begin{proof}
We begin by demonstrating the bound of Equation~\eqref{lem:sum_c:1}:

\begin{align}
    \sum_{i=1}^k \mathbb{I}(i \in \ktyp) &\sum_{j=h}^{H-1} \big[ \Pest_i(\cdot | x_{i,j}, a_{i,j}^{\pi_i}) - P(\cdot | x_{i,j}, a_{i,j}^{\pi_i}) \big]^\transpose \Vstar{j+1}(\cdot) \nonumber \\
    &\le \sum_{i=1}^k \mathbb{I}(i \in \ktyp) \sum_{j=h}^{H-1} \left[ \sqrt{\frac{2 \VVopt{j+1}(x_{i,j}, a_{i,j}^{\pi_i}) L}{N_i (x_{i,j}, a_{i,j}^{\pi_i})}} + \frac{2HL}{3N_i(x_{i,j}, a_{i,j}^{\pi_i})} \right] \label{lem:sum_c:3} \\
    \begin{split}
    \label{lem:sum_c:4}
        &\le \sqrt{2L} \sqrt{\underbrace{\sum_{i=1}^k \sum_{j=h}^{H-1} \VVopt{j+1}(x_{i,j}, a_{i,j}^{\pi_i})}_{(a)}} \sqrt{\underbrace{\sum_{i=1}^k \mathbb{I}(i \in \ktyp) \sum_{j=h}^{H-1} \frac{1}{N_i (x_{i,j}, a_{i,j}^{\pi_i})}}_{(b)}} \\
        &\quad + \underbrace{\sum_{i=1}^k \mathbb{I}(i \in \ktyp) \sum_{j=h}^{H-1} \frac{2HL}{3 N_i (x_{i,j}, a_{i,j}^{\pi_i})}}_{(c)},
    \end{split}
\end{align}

where Equation~\eqref{lem:sum_c:3} is obtained by applying Bernstein's inequality \citep[see, \eg][]{cesabianchi2006}, and Equation~\eqref{lem:sum_c:4} is obtained by applying Cauchy-Schwarz's inequality. We now bound terms $(a)$, $(b)$, and $(c)$.

By adding and subtracting $\VVpi[i]{j+1}(x_{i,j},a_{i,j}^{\pi_i})$ to term $(a)$, we can rewrite it as:

\begin{equation*}
    (a) = \underbrace{\sum_{i=1}^k \sum_{j=h}^{H-1} \VVpi[i]{j+1}(x_{i,j}, a_{i,j}^{\pi_i})}_{(d)} + \underbrace{\sum_{i=1}^k \sum_{j=h}^{H-1} \left( \VVopt{j+1}(x_{i,j},a_{i,j}^{\pi_i}) -  \VVpi[i]{j+1} (x_{i,j},a_{i,j}^{\pi_i}) \right)}_{(e)}.
\end{equation*}

As events $\mathcal{E}$ and $\Omega_{k,h}$ hold, we can apply Lemmas~\ref{lem:sum_next_state_var} and \ref{lem:sum_var_diff_opt} to bound terms $(d)$ and $(e)$, respectively, thus obtaining: 

\begin{align}
    (a) &\le H T_k + 2 H^2 \sqrt{T_k L} + \frac{4}{3} H^3 L + 2 H \sum_{i=1}^k \sum_{j=h}^{H-1} \Vdifftil{i,j}(x_{i,j}) + 4 H^2 \sqrt{T_k L} \nonumber \\
    &\le 3 T_k H + 2 H \sum_{i=1}^k \sum_{j=h}^{H-1} \Vdifftil{i,j}(x_{i,j}), \label{lem:sum_c:5}
\end{align}

where Equation~\eqref{lem:sum_c:5} holds under the condition of $\ktyp$.

We now bound terms $(b)$ and $(c)$ as follows:

\begin{align}
    (b) &\le \sum_{x \in \Ss} \sum_{a \in \As} \sum_{n=1}^{kH} n^{-1} \nonumber \\
    &\le SAL, \label{lem:sum_c:6} \\
    & \nonumber \\
    (c) &\le \frac{2}{3} HL \sum_{x \in \Ss} \sum_{a \in \As} \sum_{n=1}^{kH} n^{-1} \nonumber \\
    &\le \frac{2}{3} HSAL^2. \label{lem:sum_c:7}
\end{align}

Finally, by plugging the results of Equations~\eqref{lem:sum_c:5}, \eqref{lem:sum_c:6}, and \eqref{lem:sum_c:7} into Equation~\eqref{lem:sum_c:4}, we get:

\begin{align}
    \sum_{i=1}^k \mathbb{I}(i \in \ktyp) &\sum_{j=h}^{H-1} \big[ \Pest_i(\cdot | x_{i,j}, a_{i,j}^{\pi_i}) - P(\cdot | x_{i,j}, a_{i,j}^{\pi_i}) \big]^\transpose \Vstar{j+1}(\cdot) \nonumber \\
    &\le \sqrt{2L}\sqrt{3 T_k H + 2 H \sum_{i=1}^k \sum_{j=h}^{H-1} \Vdifftil{i,j}(x_{i,j})}\sqrt{SAL} + \frac{2}{3} HSAL^2 \nonumber \\
    &\le \sqrt{6 H S A T_k L^2} + 2\sqrt{H S A L^2 (\sum_{i=1}^k \sum_{j=h}^{H-1} \Vdifftil{i,j}(x_{i,j}))} + \frac{2}{3} HSAL^2, \label{lem:sum_c:8}
\end{align}

where Equation~\eqref{lem:sum_c:8} is obtained by computing the product of the square roots and by the subadditivity of the square root.
Following the same procedure, we can obtain the upper bound of Equation~\eqref{lem:sum_c:2} by substituting terms $T_k$ with $H N'_{k,h}(x)$, and terms $\sum_{i=1}^k \sum_{j=h}^{H-1} \Vdifftil{i,j}(x_{i,j})$ with $\sum_{i=1}^k \mathbb{I}(x_{i,h}=x) \sum_{j=h}^{H-1} \Vdifftil{i,j}(x_{i,j})$.
\end{proof}


\begin{lemma}[Summation over typical episodes of bonus terms]
\label{lem:sum_b}
Let $k \in \dsb{K}$ and $h \in \dsb{H}$. Let $\pi_k$ be the policy followed during episode $k$. Let the UCB bonus be defined as:

\begin{align*}
    \bonus_{k,h}(x,a) &= \sqrt{\frac{4 L \Var_{y \sim \Pest_k (\cdot | x,a)} [\Vest{k,h+1}(y)]}{N_k(x,a)}} + \frac{7HL}{3 (N_k(x,a) - 1)} \\
    &\quad + \sqrt{\frac{4 \mathbb{E}_{y \sim \Pest_k(\cdot|x,a)} [\min\{ \frac{84^2H^3S^2AL^2}{N'_{k,h+1}(y)} , H^2\}]}{N_k (x,a)}}.
\end{align*}

Under the events $\mathcal{E}$ and $\Omega_{k,h}$ the following inequalities hold for every $x \in \Ss$:

\begin{align}
    \begin{split}
    \label{lem:sum_b:1}
    \sum_{i=1}^k \mathbb{I}(i \in \ktyp) &\sum_{j=h}^{H-1} \bonus_{i,j}(x_{i,j}, a_{i,j}^{\pi_i}) \\
    &\le \sqrt{28 HSAT_kL^2} + \frac{7}{3} HSAL^2 + 2 \sqrt{84^2 H^3 S^4 A^2 L^4}\\
    &\quad + \sqrt{8 HSAL^2 \sum_{i=1}^k \sum_{j=h}^{H-1} \Vdifftil{i,j+1}(x_{i,j+1})},
    \end{split} \\
    \begin{split}
    \label{lem:sum_b:2}
    \sum_{i=1}^k \mathbb{I}(i \in \ktypx, &x_{i,h}=x) \sum_{j=h}^{H-1} \bonus_{i,j}(x_{i,j}, a_{i,j}^{\pi_i}) \\
    &\le \sqrt{28 H^2SA N'_{k,h}(x) L^2} + \frac{7}{3} HSAL^2 + 2 \sqrt{84^2 H^3 S^4 A^2 L^4} \\
    &\quad + \sqrt{8 HSAL^2 \sum_{i=1}^k \mathbb{I}(x_{i,h}=x) \sum_{j=h}^{H-1} \Vdifftil{i,j+1}(x_{i,j+1})},
    \end{split}
\end{align}

where:

\begin{align*}
    \ktyp &\coloneqq \{ i\in\dsb{k} : (x_{i,h}, a_{i,h}^{\pi_i}) \in [(x,a)]_k, i \ge 250 H S^2 A L, \forall h \in \dsb{H} \}, \\
    \ktypx &\coloneqq \{ i \in \dsb{k} : (x_{i,h}, a_{i,h}^{\pi_i}) \in [(x,a)]_k, N'_{k,h}(x) \ge 250 H S^2 A L, \forall h \in \dsb{H} \}, \\
    [(x,a)]_k &\coloneqq \{ (x,a) \in \Ss \times \As : N_k (x,a) \ge H, N'_{k,h}(x) \ge H, \forall h \in \dsb{H} \}.
\end{align*}
\end{lemma}

\begin{proof}
We begin by demonstrating the bound of Equation~\eqref{lem:sum_b:1}. We can rewrite the summation as:

\begin{align}
    \begin{split}
    \label{lem:sum_b:3}
        \sum_{i=1}^k \mathbb{I}(i \in \ktyp) \sum_{j=h}^{H-1} \bonus_{i,j}(x_{i,j}, a_{i,j}^{\pi_i}) &\le \underbrace{\sum_{i=1}^k \mathbb{I}(i \in \ktyp) \sum_{j=h}^{H-1} \sqrt{\frac{4 L \VVest{i,j+1}(x_{i,j}, a_{i,j}^{\pi_i})}{N_i(x_{i,j},a_{i,j}^{\pi_i})}}}_{(a)} \\
        &\quad + \underbrace{\sum_{i=1}^k \mathbb{I}(i \in \ktyp) \sum_{j=h}^{H-1} \frac{7HL}{3 (N_i(x_{i,j},a_{i,j}^{\pi_i}) - 1)}}_{(b)} \\
        &\quad + \underbrace{\sum_{i=1}^k \mathbb{I}(i \in \ktyp) \sum_{j=h}^{H-1} \sqrt{\frac{4 \mathbb{E}_{y \sim \Pest_i(\cdot|x_{i,j},a_{i,j}^{\pi_i})} \bonus'_{i,j+1}(y)}{N_i(x_{i,j}, a_{i,j}^{\pi_i})}}}_{(c)},
    \end{split}
\end{align}

where $\bonus'_{i,j+1}(y) \coloneqq \min \{ \frac{84^2 H^2 S^2 A L^2}{N'_{i,j+1}(y)}, H^2 \}$. First of all, we observe that we can bound term $(b)$ by using a pigeonhole argument as:

\begin{equation}
\label{lem:sum_b:4}
    (b) \le \frac{7}{3} HSAL^2.
\end{equation}

We now bound term $(a)$. By applying Cauchy-Schwarz's inequality, we obtain:

\begin{equation}
\label{lem:sum_b:5}
    (a) \le \sqrt{4L} \sqrt{\underbrace{\sum_{i=1}^k \sum_{j=h}^{H-1} \VVest{i,j+1}(x_{i,j}, a_{i,j}^{\pi_i})}_{(d)}} \sqrt{\underbrace{\sum_{i=1}^k \mathbb{I}(i \in \ktyp) \sum_{j=h}^{H-1} \frac{1}{N_i (x_{i,j}, a_{i,j}^{\pi_i})}}_{(e)}}.
\end{equation}

By applying the same argument as that of Equation~\eqref{lem:sum_c:6} of Lemma~\ref{lem:sum_c}, we bound term $(e)$ with $SAL$.

We can rewrite term $(d)$ as follows:

\begin{equation}
\label{lem:sum_b:6}
    (d) = \underbrace{\sum_{i=1}^k \sum_{j=h}^{H+1} \VVpi[i]{j+1}(x_{i,j}, a_{i,j}^{\pi_i})}_{(f)} + \underbrace{\sum_{i=1}^k \sum_{j=h}^{H-1} [\VVest{i,j+1} (x_{i,j}, a_{i,j}^{\pi_i}) - \VVpi[i]{j+1}(x_{i,j}, a_{i,j}^{\pi_i})]}_{(g)}
\end{equation}

Under events $\mathcal{E}$ and $\Omega_{k,h}$, we can apply Lemmas~\ref{lem:sum_next_state_var} and \ref{lem:sum_var_diff_est} to upper bound terms $(f)$ and $(g)$ respectively, obtaining the following:

\begin{align*}
    (f) &\le HT_k + 2 \sqrt{H^4 T_k L} + \frac{4 H^3 L}{3}, \\
    (g) &\le 7 H^2 S \sqrt{A T_k L} + 2 H \sum_{i=1}^k \sum_{j=h}^{H-1} \Vdifftil{i,j+1}(x_{i,j+1}).
\end{align*}

Plugging the bounds of $(f)$ and $(g)$ into Equation~\eqref{lem:sum_b:6}, we get:

\begin{align}
    (d) &\le HT_k + 2 \sqrt{H^4 T_k L} + \frac{4 H^3 L}{3} + 7 H^2 S \sqrt{A T_k L} + 2 H \sum_{i=1}^k \sum_{j=h}^{H-1} \Vdifftil{i,j+1}(x_{i,j+1}) \nonumber \\
    &\le 4 H T_k + 2 H \sum_{i=1}^k \sum_{j=h}^{H-1} \Vdifftil{i,j+1}(x_{i,j+1}), \label{lem:sum_b:7}
\end{align}

where Equation~\eqref{lem:sum_b:7} holds under the condition of $\ktyp$.
Combining the bounds of terms $(d)$ and $(e)$, we can rewrite Equation~\eqref{lem:sum_b:5} as:

\begin{align}
    (a) &\le \sqrt{4L} \sqrt{4 H T_k + 2 H \sum_{i=1}^k \sum_{j=h}^{H-1} \Vdifftil{i,j+1}(x_{i,j+1})} \sqrt{SAL} \nonumber \\
    &\le \sqrt{16 H S A T_k L^2} + \sqrt{8 HSAL^2 \sum_{i=1}^k \sum_{j=h}^{H-1} \Vdifftil{i,j+1}(x_{i,j+1})}, \label{lem:sum_b:8}
\end{align}

where Equation~\eqref{lem:sum_b:8} is obtained by expanding the products and applying the subadditivity of the square root.

To bound term $(c)$, we apply Cauchy-Schwarz's inequality, obtaining:

\begin{equation}
\label{lem:sum_b:9}
    (c) \le 2 \sqrt{\underbrace{\sum_{i=1}^k \sum_{j=h}^{H-1} \mathbb{E}_{y \sim \Pest_i(\cdot|x_{i,j},a_{i,j}^{\pi_i})} \bonus'_{i,j+1}(y)}_{(h)}} \sqrt{\underbrace{\sum_{i=1}^k \mathbb{I}(i \in \ktyp) \sum_{j=h}^{H-1} \frac{1}{N_i (x_{i,j}, a_{i,j}^{\pi_i})}}_{(i)}}.
\end{equation}

Similar to term $(e)$, we can bound term $(i)$ with $SAL$.
We now bound term $(h)$. We can rewrite the term as:

\begin{align}
    (h) &= \sum_{i=1}^k \sum_{j=h}^{H-1} \sum_{y \in \Ss} \Pest_i (y | x_{i,j}, a_{i,j}^{\pi_i}) \bonus'_{i,j+1}(y) \nonumber \\
    &= \underbrace{\sum_{i=1}^k \sum_{j=h}^{H-1} \sum_{y \in \Ss} (\Pest_i (y | x_{i,j}, a_{i,j}^{\pi_i}) - P (y | x_{i,j}, a_{i,j}^{\pi_i} )) \bonus'_{i,j+1}(y)}_{(j)} + \sum_{i=1}^k \sum_{j=h}^{H-1} \sum_{y \in \Ss} P(y | x_{i,j}, a_{i,j}^{\pi_i}) \bonus'_{i,j+1}(y) \nonumber \\
    &= (j) + \underbrace{\sum_{i=1}^k \sum_{j=h}^{H-1} \mathbb{E}_{y \sim P(\cdot | x_{i,j}, a_{i,j}^{\pi_i})} \bonus'_{i,j+1}(y) - \bonus'_{i,j+1}(x_{i,j+1})}_{(k)} + \underbrace{\sum_{i=1}^k \sum_{j=h}^{H-1} \bonus'_{i,j+1}(x_{i,j+1})}_{(l)}. \label{lem:sum_b:10}
\end{align}

We bound term $(j)$ as follows:

\begin{align}
    (j) &\le H^2 \sum_{i=1}^k \sum_{j=h}^{H-1} \| \Pest_i (\cdot | x_{i,j}, a_{i,j}^{\pi_i}) - P(\cdot | x_{i,j}, a_{i,j}^{\pi_i}) \|_1 \label{lem:sum_b:11} \\
    &\le 2 H^2 \sqrt{SL} \sum_{i=1}^k \sum_{j=h}^{H-1} (N_i (x_{i,j}, a_{i,j}^{\pi_i}))^{-1/2} \label{lem:sum_b:12} \\
    &\le 2 H^2 \sqrt{SL} \sum_{x \in \Ss} \sum_{a \in \As} \sum_{n=1}^{\frac{kH}{SA}} n^{-1/2} \nonumber \\
    &\le H^2 S \sqrt{A T_k L}, \label{lem:sum_b:13}
\end{align}

where Equation~\eqref{lem:sum_b:11} is obtained by bounding $\bonus'_{i,j+1}(y)$ with $H^2$, Equation~\eqref{lem:sum_b:12} follows by applying the result of Theorem~{2.1} of \cite{weissman2003}, which holds under event $\mathcal{E}$, and Equation~\eqref{lem:sum_b:13} follows from a derivation similar to that of term $(a)$ of Lemma~\ref{lem:sum_var_diff_est}.

To bound term $(k)$, we first observe that it is a Martingale difference sequence, and as such we can bound it via the event $\mathcal{E}_{\mathrm{az}}(\mathcal{F}_{\bonus', k, h}, H^2, L)$, which holds under $\mathcal{E}$, obtaining:

\begin{equation*}
    (k) \le 2 H^2 \sqrt{T_k L}.
\end{equation*}

By applying the definition of $\bonus'$, we can bound term $(l)$ as:

\begin{align*}
    (l) &\le 84^2 H^3 S^2 A L^2 \sum_{i=1}^k \sum_{j=h}^{H-1} \frac{1}{N'_{i,j+1}(x_{i,j+1})} \\
    &\le 84^2 H^3 S^2 A L^2 \sum_{x \in \Ss} \sum_{n=1}^T n^{-1} \\
    &\le 84^2 H^3 S^3 A L^3.
\end{align*}

Plugging the bounds of terms $(j)$, $(k)$, and $(l)$ into Equation~\eqref{lem:sum_b:10}, we get:

\begin{align*}
    (h) &\le H^2 S \sqrt{A T_k L} + 2 H^2 \sqrt{T_k L} + 84^2 H^3 S^3 A L^3.
\end{align*}

By applying the bounds of terms $(h)$ and $(i)$ to Equation~\eqref{lem:sum_b:9}, we get:

\begin{align}
    (c) &\le 2 \sqrt{H^2 S \sqrt{A T_k L} + 2 H^2 \sqrt{T_k L} + 84^2 H^3 S^3 A L^3} \sqrt{SAL} \nonumber \\
    &\le 2 \sqrt{3 H S A T_k L} + 2 \sqrt{84^2 H^3 S^4 A^2 L^4}, \label{lem:sum_b:14}
\end{align}

where Equation~\eqref{lem:sum_b:14} is obtained by expanding the products, applying the subadditivity of the square root, and applying the definition of $\ktyp$.

Finally, we can combine the bounds of terms $(a)$, $(b)$, and $(c)$ into Equation~\eqref{lem:sum_b:3}, obtaining the following bound:

\begin{align*}
    \sum_{i=1}^k \mathbb{I}(i \in \ktyp) &\sum_{j=h}^{H-1} \bonus_{i,j}(x_{i,j}, a_{i,j}^{\pi_i}) \\
    &\le \sqrt{16 H S A T_k L^2} + \sqrt{8 HSAL^2 \sum_{i=1}^k \sum_{j=h}^{H-1} \Vdifftil{i,j+1}(x_{i,j+1})} \\
    &\quad + \frac{7}{3} HSAL^2 + 2 \sqrt{3 H S A T_k L} + 2 \sqrt{84^2 H^3 S^4 A^2 L^4} \\
    &\le \sqrt{28 HSAT_kL^2} + \frac{7}{3} HSAL^2 \\
    &\quad + \sqrt{8 HSAL^2 \sum_{i=1}^k \sum_{j=h}^{H-1} \Vdifftil{i,j+1}(x_{i,j+1})} + 2 \sqrt{84^2 H^3 S^4 A^2 L^4},
\end{align*}

thus demonstrating the result of Equation~\eqref{lem:sum_b:1}. By following the same procedure, substituting $T_k$ with $H N'_{k,h}(x)$ and $\sum_{i=1}^k \sum_{j=h}^{H-1} \Vdifftil{i,j}(x_{i,j})$ with $\sum_{i=1}^k \mathbb{I}{(x_{i,h}=x)} \sum_{j=h}^{H-1} \Vdifftil{i,j}(x_{i,j})$, we can obtain an upper bound to Equation~\eqref{lem:sum_b:2} as:

\begin{align*}
    \sum_{i=1}^k \mathbb{I}(i \in \ktypx, &x_{i,h}=x) \sum_{j=h}^{H-1} \bonus_{i,j}(x_{i,j}, a_{i,j}^{\pi_i}) \\
    &\le \sqrt{28 H^2SA N'_{k,h}(x) L^2} + \frac{7}{3} HSAL^2 \\
    &\quad + \sqrt{8 HSAL^2 \sum_{i=1}^k \mathbb{I}(x_{i,h}=x) \sum_{j=h}^{H-1} \Vdifftil{i,j+1}(x_{i,j+1})} + 2 \sqrt{84^2 H^3 S^4 A^2 L^4},
\end{align*}

thus concluding the proof.

\end{proof}
\section{Proof of Theorem~\ref{thr:ucbvichUB}}
\label{apx:proof_CH}

\ucbvichUB*

We begin the proof by demonstrating optimism under the \ucbvich algorithm (\ie every optimistic value function is an upper bound of the true optimal value function), which requires us to show that, with high probability, the event $\Omega \coloneqq \{ \Vest{k,h}(x) \ge \Vstar{h}(x), \forall k \in \dsb{K}, h \in \dsb{H}, x \in \Ss \}$ holds.

\begin{lemma}[Optimism under Chernoff-Hoeffding bonus]
\label{lem:CH_opt}
    Let the optimistic bonus be defined as:
    \begin{equation*}
        \bonus_{k,h}(x,a) = \frac{2HL}{\sqrt{N_k(x,a)}}.
    \end{equation*}
    Then, under event $\mathcal{E}$, the following event holds:

    \begin{equation*}
        \Omega \coloneqq \{ \Vest{k,h}(x) \ge \Vstar{h}(x), \forall k \in \dsb{K}, h \in \dsb{H}, x \in \Ss \}.
    \end{equation*}
\end{lemma}

\begin{proof}
    We demonstrate the result by induction. Let $\Vest{k,h}$ be the optimistic value function at stage $h$ computed using the history up to the end of episode $k-1$, and let $\Vstar{h}$ be the true optimal value function at stage $h$.

    By definition, $\Vest{k,H+1}(x) = \Vstar{H+1}(x) = 0$ for every $x \in \Ss$, and thus the inequality $\Vest{k,H+1} \ge \Vstar{H+1}$ trivially holds. To prove the inductive step, we need to demonstrate that, if $\Vest{k, h+1} \ge \Vstar{h+1}$ holds, then it also holds that $\Vest{k,h} \ge \Vstar{h}$. We can derive this result as follows:

    \begin{align}
        \Vest{k,h}(x) - \Vstar{h} &= \max_{a\in\As} \Qest{k,h}(x,a) - \Vstar{h}(x) \nonumber \\
        &\ge \Qest{k,h}(x, a_{k,h}^{\pi^*}) - \Vstar{h}(x) \nonumber \\
        &= \sum_{y\in\Ss} \Pest_{k}(y | x, a_{k,h}^{\pi^*})\Vest{k,h+1}(y) + \bonus_{k,h}(x, a_{k,h}^{\pi^*}) - \sum_{y\in\Ss} P(y | x, a_{k,h}^{\pi^*})\Vstar{h+1}(y) \nonumber \\
        &\ge \sum_{y \in \Ss} \left[ \Pest_{k}(y | x, a_{k,h}^{\pi^*}) - P(y | x, a_{k,h}^{\pi^*})\right] \Vstar{h+1}(y)  + \bonus_{k,h}(x, a_{k,h}^{\pi^*}) \label{lem:CH_opt:1} \\
        &\ge \bonus_{k,h}(x, a_{k,h}^{\pi^*}) - 2\sqrt{\frac{H^2 L}{N_k(x_{k,h}, a_{k,h}^{\pi^*})}} \label{lem:CH_opt:2} \\
        &\ge 2\frac{H L}{\sqrt{N_k (x_{k,h}, a_{k,h}^{\pi^*})}} - 2\sqrt{\frac{H^2 L}{N_k (x_{k,h}, a_{k,h}^{\pi^*})}} \nonumber \\
        &\ge 0, \nonumber
    \end{align}

    where Equation~\eqref{lem:CH_opt:1} follows by the inductive hypothesis, and Equation~\eqref{lem:CH_opt:2} is obtained because, under $\mathcal{E}$, we can bound $|\Pest_{k}(y | x, a_{k,h}^{\pi^*}) - P(y | x, a_{k,h}^{\pi^*}) \Vstar{h+1}(y)|$ by applying Azuma-Hoeffding's inequality, allowing us to simplify terms and show optimism.
\end{proof}

Our objective is to bound the regret after $K$ episodes (\ie $\regch{K}$). We can observe that, under event $\Omega$, it holds that:

\begin{align*}
    \regch{K} &= \sum_{k \in \dsb{K}} \Vstar{1}(x_{k,1}) - \Vpi{1}(x_{k,1}) \\
    &\le \sum_{k \in \dsb{K}} \Vest{k,1}(x_{k,1}) - \Vpi{1}(x_{k,1}) \\
    &= \sum_{k \in \dsb{K}} \Vdifftil{k,1}(x_{k,1}) \\
    &= \regchtilde{K}.
\end{align*}

As such, we can now focus on finding an upper bound to $\regchtilde{K}$. By applying Lemma~\ref{lem:regr_dec}, we can write:

\begin{align}
    \regchtilde{K} &= \sum_{k \in \dsb{K}} \Vdifftil{k,1}(x_{k,1}) \nonumber \\
    \begin{split}
    \label{lem:CH_opt:3}
    &\le e \sum_{i=1}^K\sum_{j=1}^{H-1} \bigg[ \varepsilon_{i,j} + 2 \sqrt{L} \overline{\varepsilon}_{i,j} + \bonus_{i,j}(x_{i,j}, a_{i,j}^{\pi_i}) \\
    &\quad + \xi_{i,j} (x_{i,j}, a_{i,j}^{\pi_i}) + \frac{8H^2SL}{3 N_i (x_{i,j}, a_{i,j}^{\pi_i})} \bigg].
    \end{split}
\end{align}

To find an upper bound to the regret, we can thus bound the summation of each of the terms individually.

By applying Lemma~\ref{lem:sum_eps}, we obtain the following bounds:

\begin{align*}
    \sum_{i=1}^K \sum_{j=1}^{H-1} \varepsilon_{i,j} &\le 2 \sqrt{H^2 T L}, \\
    \sum_{i=1}^K \sum_{j=1}^{H-1} 2 \sqrt{L} \overline{\varepsilon}_{i,j} &\le 4 \sqrt{T L}. \\
\end{align*}

Then, we can derive the following bound:

\begin{align}
    \sum_{i=1}^K \sum_{j=1}^H \frac{8H^2SL}{3N_i(x_{i,j},a_{i,j}^{\pi_i})} &= \frac{8}{3}H^2SL \sum_{x \in \Ss} \sum_{a \in \As} \sum_{n=1}^{N_K(x,a)} n^{-1} \label{lem:CH_opt:5} \\
    &\le \frac{8}{3}H^2SL \sum_{x \in \Ss} \sum_{a \in \As} \sum_{n=1}^{\frac{KH}{SA}} n^{-1} \label{lem:CH_opt:6} \\
    &\le \frac{8}{3} H^2 S^2 A L^2 \nonumber
\end{align}

where Equation~\eqref{lem:CH_opt:5} is obtained by rearranging the terms to isolate the summation of $n^{-1}$ for $n$ from $1$ to $N_K(x,a)$ (\ie the total number of times each state-action pair has been observed up to the end of episode $K$), and Equation~\eqref{lem:CH_opt:6} derives from the observation that the summation can be upper bounded by considering a uniform state-action visit distribution. This derivation produces the same result as applying the well-known pigeonhole principle.

By applying a similar reasoning, we bound the remaining summations over the bonus terms:

\begin{align}
    \sum_{i=1}^K \sum_{j=1}^H \bonus_{i,j}(x_{i,j}) &= \sum_{i=1}^K \sum_{j=1}^H 2H \sqrt{\frac{L}{N_i (x_{i,j}, a_{i,j}^{\pi_i})}} \nonumber \\
    &= 2H \sqrt{L} \sum_{x\in\Ss}\sum_{a\in\As} \sum_{n=1}^{N_K(x,a)} n^{-1/2} \nonumber \\
    &\le 2H \sqrt{L} \sum_{x\in\Ss}\sum_{a\in\As} \sum_{n=1}^{\frac{KH}{SA}} n^{-1/2} \nonumber \\
    &\le 2 \sqrt{H^2 S A T L}, \nonumber
\end{align}

and over the model error terms:

\begin{align}
    \sum_{i=1}^K \sum_{j=1}^H \xi_{i,j} (x_{i,j}, a_{i,j}^{\pi_i}) &\le \sum_{i=1}^K \sum_{j=1}^H 2 H \sqrt{\frac{L}{N_i (x_{i,j}, a_{i,j}^{\pi_i})}} \label{lem:CH_opt:4} \\
    &= 2 H \sqrt{L} \sum_{x\in\Ss}\sum_{a\in\As} \sum_{n=1}^{N_K(x,a)} n^{-1/2} \nonumber \\
    &\le 2H \sqrt{L} \sum_{x\in\Ss}\sum_{a\in\As} \sum_{n=1}^{\frac{KH}{SA}} n^{-1/2} \nonumber \\
    &\le 2 \sqrt{H^2 S A T L}, \nonumber
\end{align}

where Equation~\eqref{lem:CH_opt:4} is obtained by bounding $\xi_{i,j} (x_{i,j}, a_{i,j}^{\pi_i})$ using the Chernoff-Hoeffding inequality.
Finally, we can put all the bounds together and rewrite Equation~\eqref{lem:CH_opt:3} as:

\begin{align*}
    \regchtilde{K} &\le e \left[ 2 \sqrt{H^2 T L} + 4 \sqrt{TL} + 2 \sqrt{H^2 S A T L} + 2 \sqrt{H^2 S A T L} + \frac{8}{3} H^2 S^2 A L^2 \right] \\
    &\le e \left[ 10 \sqrt{H^2 S A T L} + \frac{8}{3} H^2 S^2 A L^2 \right],
\end{align*}

thus completing the proof.
\section{Proof of Theorem~\ref{thr:ucbvibfUB}}
\label{apx:proof_BF}

\ucbvibfUB*

Similarly to the proof of Theorem~\ref{thr:ucbvichUB} in Appendix~\ref{apx:proof_CH}, in order to demonstrate the upper bound of \ucbvibf, we first need to demonstrate optimism. However, in order to remove the additional $\sqrt{H}$ term, we are required to both demonstrate optimism as well as to bound by how much the optimistic value function estimator exceeds the true optimal value function.

We start by observing that:

\begin{align*}
    \regbf{K} &= \sum_{k \in \dsb{K}} \Vstar{1}(x_{k,1}) - \Vpi{1}(x_{k,1}) \\
    &\le \sum_{k \in \dsb{K}} \Vest{k,1}(x_{k,1}) - \Vpi{1}(x_{k,1}) \\
    &= \sum_{k \in \dsb{K}} \Vdifftil{k,1}(x_{k,1}) \\
    &= \regbftilde{K}.
\end{align*}

According to Lemma~\ref{lem:regr_dec}, under the events $\mathcal{E}$ and $\Omega_{k,h}$, we can decompose the pseudo-regret as:

\begin{equation}
    \sum_{i=1}^k \Vdifftil{i,h}(x_{i,h}) \le e \sum_{i=1}^k \sum_{j=h}^{H-1} \left[ \varepsilon_{i,j} + 2 \sqrt{L} \overline{\varepsilon}_{i,j} + \bonus_{i,j}(x_{i,j},a_{i,j}^{\pi_i}) + \xi_{i,j} (x_{i,j}, a_{i,j}^{\pi_i}) + \frac{8H^2SL}{3 N_i (x_{i,j}, a_{i,j}^{\pi_i})} \right]. \label{thr:ucbvibfUB:1}
\end{equation}

We also define, by trivially modifying the derivation of Lemma~\ref{lem:regr_dec}, the pseudo-regret considering only the episodes in which, at stage $h \in \dsb{H}$ a specific state $x \in \Ss$ was occupied:

\begin{align}
    \sum_{i=1}^k \mathbb{I}(x_{i,h} = x) \Vdiff{i,h}(x_{i,h}) &\le \sum_{i=1}^k \mathbb{I}(x_{i,h} = x) \Vdifftil{i,h}(x_{i,h}) \nonumber \\
    \begin{split}
    \label{thr:ucbvibfUB:2}
    &\le e \sum_{i=1}^k \mathbb{I}(x_{i,h} = x) \sum_{j=h}^{H-1} \bigg[ \varepsilon_{i,j} + 2 \sqrt{L} \overline{\varepsilon}_{i,j} + \bonus_{i,j}(x_{i,j},a_{i,j}^{\pi_i}) \\
    &\quad + \xi_{i,j} (x_{i,j}, a_{i,j}^{\pi_i}) + \frac{8H^2SL}{3 N_i (x_{i,j}, a_{i,j}^{\pi_i})} \bigg]. 
    \end{split}
\end{align}

By applying Lemmas~\ref{lem:sum_eps} and ~\ref{lem:sum_eps_x}, we can upper bound Equations~\eqref{thr:ucbvibfUB:1} and \eqref{thr:ucbvibfUB:2} as:

\begin{align*}
    \sum_{i=1}^k \Vdifftil{i,h}(x_{i,h}) &\le e \sum_{i=1}^k \sum_{j=h}^{H-1} \left[ \bonus_{i,j}(x_{i,j},a_{i,j}^{\pi_i}) + \xi_{i,j} (x_{i,j}, a_{i,j}^{\pi_i}) + \frac{8H^2SL}{3 N_i (x_{i,j}, a_{i,j}^{\pi_i})} \right] \\
    &\quad + 2e\sqrt{H^2T_kL} + 4e\sqrt{T_kL} \\
    &= \proxyUB,
\end{align*}

and:

\begin{align*}
    \sum_{i=1}^k \mathbb{I}(x_{i,h} = x) \Vdifftil{i,h}(x_{i,h}) &\le e \sum_{i=1}^k \mathbb{I}(x_{i,h} = x) \sum_{j=h}^{H-1} \left[ \bonus_{i,j}(x_{i,j},a_{i,j}^{\pi_i}) + \xi_{i,j} (x_{i,j}, a_{i,j}^{\pi_i}) + \frac{8H^2SL}{3 N_i (x_{i,j}, a_{i,j}^{\pi_i})} \right]\\
    &\quad + 2e\sqrt{H^3 N'_{k,h}(x) L} + 4e\sqrt{H N'_{k,h}(x) L} \\
    &= \proxyUBx,
\end{align*}

where we denote the upper bounds of $\sum_{i=1}^k \Vdifftil{i,h}(x_{i,h})$ and $\sum_{i=1}^k \mathbb{I}(x_{i,h} = x) \Vdifftil{i,h}(x_{i,h})$ as $\proxyUB$ and $\proxyUBx$, respectively, for ease of notation.

We now demonstrate optimism, which requires us to show that, with high probability, the event $\Omega_{k,h}$ holds.

\begin{lemma}[Optimism under Bernstein-Freedman bonus]
\label{lem:BF_opt}
Let the optimistic bonus be defined as:

\begin{align*}
    \bonus_{k,h}(x,a) &= \sqrt{\frac{4 L \Var_{y \sim \Pest_k (\cdot | x,a)} [\Vest{k,h+1}(y)]}{N_k(x,a)}} + \frac{7HL}{3 (N_k(x,a) - 1)} \\
    &\quad + \sqrt{\frac{4 \min\{\mathbb{E}_{y \sim \Pest_k(\cdot|x,a)} [\frac{84^2H^3S^2AL^2}{N'_{k,h+1}(y)}] , H^2\}}{N_k (x,a)}}.
\end{align*}

Then, under event $\mathcal{E}$, the following set of events hold:

\begin{equation*}
    \Omega_{k,h} \coloneqq \left\{ \Vest{i,j}(x) \ge \Vstar{j}(x), \forall (i,j) \in [k,h]_{\mathrm{hist}}, x \in \Ss \right\},
\end{equation*}

for $k \in \dsb{K}$ and $h \in \dsb{H}$, where:

\begin{equation*}
    [k,h]_{\mathrm{hist}} \coloneqq \left\{ (i,j) \in \dsb{K} \times \dsb{H} : i < k \vee (i=h, j \ge h) \right\}.
\end{equation*}
\end{lemma}

\begin{proof}
We demonstrate the result by induction. We begin by observing that $\Vest{k, H+1}(x) = \Vstar{H+1}(x) = 0$ for every $k \in \dsb{K}$ and $x \in \Ss$.
To prove the induction, we need to prove that, if $\Omega_{k,h}$ holds, then also $\Omega_{k,h-1}$ holds. We prove this for a generic $k \in \dsb{K}$, and we can then apply this procedure for increasing values of $k$, starting from $k=1$.

If $\Omega_{k,h}$ holds, then $\Vest{k,h}(x) \ge \Vstar{h}(x)$ for every $x \in \Ss$. We now bound the estimation error due to the optimistic approach:

\begin{align}
    \Vest{k,h}(x) - \Vstar{h}(x) &= \frac{1}{N'_{k,h}(x)} \sum_{i=1}^k \mathbb{I}(x_{i,h}=x) (\Vest{k,h}(x) - \Vstar{h}(x)) \nonumber \\
    &\le \frac{1}{N'_{k,h}(x)} \sum_{i=1}^k \mathbb{I}(x_{i,h}=x) (\Vest{i,h}(x) - \Vpi[i]{h}(x)) \label{lem:BF_opt:1} \\
    &= \frac{1}{N'_{k,h}(x)} \sum_{i=1}^k \mathbb{I}(x_{i,h} = x) \Vdifftil{i,h}(x_{i,h}), \label{lem:BF_opt:2}
\end{align}

where Equation~\eqref{lem:BF_opt:1} follows from the fact that $\Vest{k,h}$ is monotonically decreasing in $k$ by definition, and by observing that $\Vstar{h} \ge \Vpi[i]{h}$.

Recalling the upper bound of $\sum_{i=1}^k \mathbb{I}(x_{i,h} = x) \Vdifftil{i,h}(x_{i,h})$:

\begin{align}
    \sum_{i=1}^k \mathbb{I}(x_{i,h} = x) \Vdifftil{i,h}(x_{i,h}) &\le \proxyUBx \nonumber \\
    &= e \sum_{i=1}^k \mathbb{I}(x_{i,h} = x) \sum_{j=h}^{H-1} \bigg[ \bonus_{i,j}(x_{i,j},a_{i,j}^{\pi_i}) + \xi_{i,j} (x_{i,j}, a_{i,j}^{\pi_i}) \nonumber \\ 
    &\quad + \frac{8H^2SL}{3 N_i (x_{i,j}, a_{i,j}^{\pi_i})} \bigg] + 2e\sqrt{H^3 N'_{k,h}(x) L} + 4e\sqrt{H N'_{k,h}(x) L}, \nonumber
\end{align}

we now bound the summations over the terms in the summation over episodes and stages. By applying Lemma~\ref{lem:sum_b}, we can bound the summation over typical episodes of the bonus terms as:

\begin{align*}
    \sum_{i=1}^k \mathbb{I}(i \in \ktypx, x_{i,h}=x) \sum_{j=h}^{H-1} \bonus_{i,j}(x_{i,j}, a_{i,j}^{\pi_i}) &\le \sqrt{28 H^2SA N'_{k,h}(x) L^2} + \frac{7}{3} HSAL^2 \\
    &\quad + 2 \sqrt{84^2 H^3 S^4 A^2 L^4} + \sqrt{8 H^2 SAL^2 \proxyUBx},
\end{align*}

by observing that $\sum_{i=1}^k \mathbb{I}(x_{i,h} = x) \Vdiff{i,h}(x_{i,h}) \le \proxyUBx$ and that the series of $\proxyUBx$ terms is decreasing in $h$, as each term $\proxyUBx$ is a summation of elements each of which includes the next term, and as such we can upper bound $\sum_{j=h}^{H-1} U_{k,j,x}$ with $H \proxyUBx$.

In a similar way, we can apply the result of Lemma~\ref{lem:sum_c} to bound the summation over typical episodes of the state-action wise model error terms as:

\begin{align*}
    \sum_{i=1}^k \mathbb{I}(i \in \ktypx, x_{i,h}=x) &\sum_{j=h}^{H-1} \big[ \Pest_i(\cdot | x_{i,j}, a_{i,j}^{\pi_i}) - P(\cdot | x_{i,j}, a_{i,j}^{\pi_i}) \big]^\transpose \Vstar{j+1}(\cdot) \\
    &\le \sqrt{6 H^2 S A N'_{k,h}(x) L^2} + \frac{2}{3} HSAL^2 + 2\sqrt{H^2 S A L^2 \proxyUBx}.
\end{align*}

With the same procedure as in the proof of Lemma~\ref{lem:CH_opt}, we obtain the following upper bound:

\begin{equation}
\label{lem:BF_opt:9}
    \sum_{i=1}^K \sum_{j=1}^H \frac{8H^2SL}{3N_i(x_{i,j},a_{i,j}^{\pi_i})} \le \frac{8}{3} H^2 S^2 A L^2.
\end{equation}

By combining these result, and accounting for the regret on non-typical episodes, we can write:

\begin{align}
    \sum_{i=1}^k \mathbb{I}(x_{i,h} = x) \Vdiff{i,h}(x_{i,h}) &\le \proxyUBx \nonumber \\
    &\le e \bigg[ \sqrt{28 H^2 S A N'_{k,h}(x) L^2} + \frac{7}{3} HSAL^2 + 168 \sqrt{H^3 S^4 A^2 L^4} \nonumber \\
    &\quad + \sqrt{8 H^2 S A L^2 \proxyUBx} + \sqrt{6 H^2 S A N'_{k,h}(x) L^2} + \frac{2}{3}HSAL^2 \nonumber \\
    &\quad + 2\sqrt{H^2 S A L^2 \proxyUBx} + \frac{8}{3} H^2 S^2 A L^2 + 2 \sqrt{H^3 N'_{k,h}(x) L} \nonumber \\
    &\quad + 4 \sqrt{H N'_{k,h}(x)L} + 100 H^2 S^2 A L^2 \bigg] \nonumber \\
    &\le e \bigg[ 12 \sqrt{H^2 S A N'_{k,h}(x) L^2} + 5 \sqrt{H^2 S A L^2 \proxyUBx} \nonumber \\
    &\quad + \frac{821}{3} H^2 S^2 A L^2 + 2 \sqrt{H^3 N'_{k,h}(x) L}\bigg]. \nonumber
\end{align}

Letting:

\begin{align*}
    \alpha &= e \left[ 12 \sqrt{H^2 S A N'_{k,h}(x) L^2} + \frac{821}{3} H^2 S^2 A L^2 + 2 \sqrt{H^3 N'_{k,h}(x) L}\right], \\
    \beta &= 5e\sqrt{H^2 S^2 A L^2},
\end{align*}

we can solve for $\proxyUBx$ and obtain the following upper bound:

\begin{equation*}
    \proxyUBx \le \beta^2 + 2\alpha,
\end{equation*}

which we can write as:

\begin{align}
    \proxyUBx &\le 25 e^2 H^2 S^2 A L^2 + 24 e \sqrt{H^2 S A N'_{k,h}(x) L^2} + \frac{1642}{3} e H^2 S^2 A L^2 + 4 e \sqrt{H^3 N'_{k,h}(x) L} \nonumber \\
    &\le 24 e \sqrt{H^2 S A N'_{k,h}(x) L^2} + \frac{1846}{3} e H^2 S^2 A L^2 + 4 e \sqrt{H^3 N'_{k,h}(x) L} \nonumber \\
    &\le 28 e \sqrt{H^2 S A N'_{k,h}(x) L^2} + \frac{1846}{3} e H^2 S^2 A L^2 \label{lem:BF_opt:3} \\
    &\le 28 \cdot \frac{12}{11} e \sqrt{H^3 S^2 A N'_{k,h}(x) L^2} \label{lem:BF_opt:4} \\
    &\le 84 \sqrt{H^3 S^2 A N'_{k,h}(x) L^2}, \nonumber
\end{align}

where Equation~\eqref{lem:BF_opt:3} holds if $SA \ge H$, and Equation~\eqref{lem:BF_opt:4} holds under the condition of $\ktypx$.

Plugging this result into Equation~\eqref{lem:BF_opt:2}, and observing that the error cannot be greater than $H$, we get the following upper bound to the estimation error due to the optimistic approach:

\begin{equation}
\label{lem:BF_opt:5}
    \Vest{k,h}(x) - \Vstar{h}(x) \le \min \left\{ 84 \sqrt{\frac{H^3 S^2 A L^2}{N'_{k,h}(x)}}, H \right\}.
\end{equation}

Using this result, we now prove that $\Vest{k,h-1}(x) \ge \Vstar{h-1}(x)$. Let us recall the definition of $\Vest{k,h-1}(x)$:

\begin{equation*}
    \Vest{k,h-1}(x) = \min \left\{ \Vest{k-1, h-1}(x), H, \mathcal{T}_{h-1}^{\pi_k} \Vest{k,h} \right\},
\end{equation*}

where $\mathcal{T}_{h-1}^{\pi_k} \Vest{k,h} \coloneqq R^{\pi_k}(x_{k,h-1}) + \bonus_{k, h-1}(x_{k,h-1}, a_{k,h-1}^{\pi_k}) + \mathbb{E}_{y \sim \Pest_k(\cdot|x_{k,h-1}, a_{k,h-1}^{\pi_k})} \Vest{k,h}(y)$.
Observe that, if $\Vest{k, h-1}(x) = H$, then the optimism holds trivially. Also, if $\Vest{k, h-1}(x) = \Vest{k-1, h-1}(x)$, the optimism holds trivially under $\Omega_{k,h}$. As such, we only need to demonstrate the case in which $\Vest{k,h-1}(x) = \mathcal{T}_{h-1}^{\pi_k} \Vest{k,h}$. As such, we derive the following:

\begin{align}
    \Vest{k,h-1}(x) &- \Vstar{h-1}(x) \nonumber \\
    &= \max_{a \in \As} \left\{ R(x, a) + \bonus_{k,h-1}(x,a) + \sum_{y \in \Ss} \Pest_k (y | x, a) \Vest{k, h}(y) \right\} \nonumber \\
    &\quad - R(x, a_{k,h-1}^{\pi^*}) - \sum_{y \in \Ss} P(y | x, a_{k,h-1}^{\pi^*}) \Vstar{h}(y) \nonumber \\
    &\ge  \bonus_{k, h-1}(x, a_{k,h-1}^{\pi^*}) + \sum_{y \in \Ss} \Pest_k (y | x,a_{k,h-1}^{\pi^*}) \Vest{k, h}(y) \nonumber \\
    &\quad - \sum_{y \in \Ss} P(y | x, a_{k,h-1}^{\pi^*}) \Vstar{h}(y) \nonumber \\
    &= \bonus_{k, h-1}(x, a_{k,h-1}^{\pi^*}) + \sum_{y \in \Ss} \Pest_k (y | x,a_{k,h-1}^{\pi^*}) \left[ \Vest{k,h}(y) - \Vstar{h}(y) \right] \nonumber \\
    &\quad + \sum_{y \in \Ss} \left[ \Pest_k (y | x,a_{k,h-1}^{\pi^*}) - P(y | x, a_{k,h-1}^{\pi^*}) \right] \Vstar{h}(y) \nonumber \\
    &\ge \bonus_{k, h-1}(x, a_{k,h-1}^{\pi^*}) + \sum_{y \in \Ss} \left[ \Pest_k (y | x,a_{k,h-1}^{\pi^*}) - P(y | x, a_{k,h-1}^{\pi^*}) \right] \Vstar{h}(y) \label{lem:BF_opt:6}
\end{align}

where Equation~\eqref{lem:BF_opt:6} follows from the induction assumption.

Under event $\mathcal{E}$, we can apply the empirical Bernstein inequality \citep{maurer2009}:

\begin{equation*}
    \left| \sum_{y \in \Ss} \left[ \Pest_k (y | x,a) - P(y | x, a) \right] \Vstar{h}(y) \right| \le \sqrt{\frac{2 \VVestopt{k,h}(x,a)L}{N_k (x,a)}} + \frac{7 HL}{3 (N_k(x,a)-1)},
\end{equation*}

where $\VVestopt{k,h}(x,a) \coloneqq \Var_{y \sim \Pest_k (\cdot | x,a)} [\Vstar{h}(y)]$. As such, we obtain:

\begin{align}
    \Vest{k,h-1}(x) - \Vstar{h-1}(x) &\ge \bonus_{k, h-1}(x, a_{k,h-1}^{\pi^*}) - \sqrt{\frac{2 \VVestopt{k,h}(x, a_{k,h-1}^{\pi^*})L}{N_k (x, a_{k,h-1}^{\pi^*})}} - \frac{7 HL}{3 (N_k(x, a_{k,h-1}^{\pi^*})-1)} \nonumber \\
    \begin{split}
    \label{lem:BF_opt:7}
    &= \sqrt{\frac{4 \VVest{k,h}(x, a_{k,h-1}^{\pi^*})L}{N_k (x, a_{k,h-1}^{\pi^*})}} + \sqrt{\frac{4L \mathbb{E}_{y \sim \Pest_k (\cdot | x, a_{k,h-1}^{\pi^*})}\bonus'_{k,h}(y)}{N_k (x, a_{k,h-1}^{\pi^*})}} \\
    &\quad - \sqrt{\frac{2 \VVestopt{k,h}(x, a_{k,h-1}^{\pi^*})L}{N_k (x,x, a_{k,h-1}^{\pi^*})}}.
    \end{split}
\end{align}

We now bound $\VVestopt{k, h}$ in terms of $\VVest{k, h}$. Observing that:

\begin{align*}
    \Var[X] &= \mathbb{E}[X - \mathbb{E}[X]]^2 \\
    &= \mathbb{E}[X \pm Y - \mathbb{E}[X] \pm \mathbb{E}[Y]]^2 \\
    &= \mathbb{E}[(X-Y) - \mathbb{E}[X-Y] + Y - \mathbb{E}[Y]]^2 \\
    &\le 2 \mathbb{E}[(X-Y) - \mathbb{E}[X-Y]]^2 + 2 \mathbb{E}[Y - \mathbb{E}[Y]]^2 \\
    &= \Var[X-Y] + 2 \Var[Y],
\end{align*}

we can then rewrite:

\begin{align*}
    \VVestopt{k, h}(x, a_{k,h-1}^{\pi^*}) &\le 2 \VVest{k, h}(x, a_{k,h-1}^{\pi^*}) + 2 \Var_{y \sim \Pest_k (\cdot | x, a_{k,h-1}^{\pi^*})} [\Vstar{h}(y) - \Vest{k,h}(y)] \\
    &\le 2 \VVest{k, h}(x, a_{k,h-1}^{\pi^*}) + 2 \sum_{y \in \Ss} \Pest_k (\cdot | x, a_{k,h-1}^{\pi^*}) (\Vest{k,h}(y) - \Vstar{y})^2.
\end{align*}

By plugging this result into Equation~\eqref{lem:BF_opt:7}, we get:

\begin{equation*}
    \Vest{k,h-1}(x) - \Vstar{h-1}(x) \ge \sqrt{\frac{4L \mathbb{E}_{y \sim \Pest_k (\cdot | x, a_{k,h-1}^{\pi^*})}\bonus'_{k,h}(y)}{N_k (x, a_{k,h-1}^{\pi^*})}} - \sqrt{\frac{4L \mathbb{E}_{y \sim \Pest_k (\cdot | x, a_{k,h-1}^{\pi^*})}(\Vest{k,h}(y) - \Vstar{h}(y))^2}{N_k (x, a_{k,h-1}^{\pi^*})}}.
\end{equation*}

By applying the result of Equation~\eqref{lem:BF_opt:5} and the definition of $\bonus'_{k,h}(y)$, we finally obtain that $\Vest{k,h-1}(x) - \Vstar{h-1}(x) \ge 0$, thus demonstrating optimism.
\end{proof}

Having demonstrated optimism, we now prove the upper bound of the regret $\regbf{K}$:

\begin{align}
    \regbftilde{K} &\le U_{K,1} \nonumber \\
    \begin{split}
    \label{lem:BF_opt:8}
    &= e \bigg[ \sqrt{28 HSATL^2} + \frac{7}{3}HSAL^2 + 2\sqrt{84^2 H^3 S^4 A^2 L^4} \\
    &\quad + \sqrt{8 HSAL^2 U_{K1}} +\sqrt{6 HSATL^2} + \frac{2}{3} HSAL^2 \\
    &\quad + 2 \sqrt{HSAL^2 U_{K,1}} + \frac{8}{3} H^2S^2AL^2 + 2\sqrt{H^2TL} \\
    &\quad + 4 \sqrt{TL} + 100 H^2 S^2 A L \bigg]
    \end{split} \\
    &\le e \bigg[ 12 \sqrt{HSATL^2} + 5 \sqrt{H^2 SAL^2 U_{K,1}} \nonumber \\
    &\quad + \frac{821}{3} H^2 S^2 A L^2 + 2 \sqrt{H^2 T L} \bigg]
\end{align}

where Equation~\eqref{lem:BF_opt:8} is obtained by applying the results of Lemmas~\ref{lem:sum_b} and \ref{lem:sum_c}, by applying the result of Equation~\eqref{lem:BF_opt:9}, and by accounting for the regret of non-typical episodes.

As done in Lemma~\ref{lem:BF_opt}, by letting:

\begin{align*}
    \alpha &= e \left[ 12 \sqrt{H S A T L^2} + \frac{821}{3} H^2 S^2 A L^2 + 2 \sqrt{H^2 T L}\right], \\
    \beta &= 5e\sqrt{H^2 S^2 A L^2},
\end{align*}

we can solve for $U_{K,1}$ and obtain:

\begin{align*}
    \regbftilde{K} &\le 24e \sqrt{HSATL^2} + \frac{1846}{3}e H^2 S^2 A L^2 + 4e \sqrt{H^2 T L} \\
    &\le 24e \sqrt{HSATL^2} + 616 e H^2 S^2 A L^2 + 4e \sqrt{H^2 T L}
\end{align*}

thus completing the proof.

\end{document}